\documentclass[twoside]{article}

\usepackage[accepted]{aistats2025}
\usepackage[colorlinks,citecolor=blue]{hyperref}

\usepackage[utf8]{inputenc} 
\usepackage{url}            
\usepackage{booktabs}       
\usepackage{amsfonts}       
\usepackage{nicefrac}       
\usepackage{microtype}      

\usepackage{times}             

\usepackage[english]{babel} 

\usepackage{amssymb}
\usepackage{mathtools}
\usepackage{amsthm}
\usepackage[table]{xcolor}

\usepackage{bbm}
\usepackage{verbatim,float,dsfont}
\usepackage{psfrag}
\usepackage{algorithm,algorithmic}
\usepackage{thmtools,thm-restate}
\usepackage{mathtools,enumitem}
\mathtoolsset{showonlyrefs}
\usepackage{tcolorbox}
\usepackage{breqn,lipsum}
\usepackage{stackengine}
\usepackage{array}
\usepackage{caption}
\usepackage{boldline}
\usepackage{amssymb}
\usepackage{multirow}
\usepackage{color}
\usepackage{colortbl}
\usepackage{makecell}

\usepackage{subfig}
\usepackage{graphicx}
\usepackage{wrapfig}
\usepackage[nocomma]{optidef}
\usepackage{mdframed}
\usepackage{subfig}
\usepackage{soul}

\usepackage{tikz}
\usetikzlibrary{shapes,arrows}

\usetikzlibrary{positioning, quotes}

\newtheorem{theorem}{Theorem}

\newtheorem{corollary}[theorem]{Corollary}
\newtheorem{proposition}[theorem]{Proposition}
\newtheorem{remark}[theorem]{Remark}
\newtheorem{definition}[theorem]{Definition}

\makeatletter
\newcommand{\setword}[2]{%
  \phantomsection
  #1\def\@currentlabel{\unexpanded{#1}}\label{#2}%
}
\makeatother

\makeatletter
\newenvironment{framework}[1][htb]{%
    \renewcommand{\ALG@name}{Framework}
   \begin{algorithm}[#1]%
  }{\end{algorithm}}
\makeatother

\makeatletter
\def\set@curr@file#1{\def\@curr@file{#1}} 
\makeatother
\usepackage[load-configurations=version-1]{siunitx} 

\def\shownotes{1}  
\ifnum\shownotes=1
\newcommand{\authnote}[2]{$\ll$\textsf{\footnotesize #1 notes: #2}$\gg$}
\else
\newcommand{\authnote}[2]{}
\fi

\usepackage{accents}

\makeatletter
\newcommand*\rel@kern[1]{\kern#1\dimexpr\macc@kerna}
\newcommand*\widebar[1]{%
  \begingroup
  \def\mathaccent##1##2{%
    \rel@kern{0.8}%
    \overline{\rel@kern{-0.8}\macc@nucleus\rel@kern{0.2}}%
    \rel@kern{-0.2}%
  }%
  \macc@depth\@ne
  \let\math@bgroup\@empty \let\math@egroup\macc@set@skewchar
  \mathsurround\z@ \frozen@everymath{\mathgroup\macc@group\relax}%
  \macc@set@skewchar\relax
  \let\mathaccentV\macc@nested@a
  \macc@nested@a\relax111{#1}%
  \endgroup
}
\makeatother

\newcommand{\argmax}{\mathop{\mathrm{argmax}}}

\def\cA{\mathcal{A}}

\def\cD{\mathcal{D}}

\def\cF{\mathcal{F}}
\def\cG{\mathcal{G}}

\def\cH{\mathcal{H}}

\def\cR{\mathcal{R}}

\def\cW{\mathcal{W}}
\def\cX{\mathcal{X}}

\def\op{\Delta}
\def\graphop{\op}

\def\kronop{\widetilde\op} 
 
\def\vitaop{\widebar\op}

\def\bs{\ensuremath\boldsymbol}

\newcommand{\red}[1]{\textcolor{red}{#1}}

\usepackage{times}

\DeclareMathSymbol{\mrq}{\mathord}{operators}{`'}
\usepackage{cite}

\begin{document}

%

%

\twocolumn[

\aistatstitle{Adapting to Online Distribution Shifts in Deep Learning: A Black-Box Approach}

\aistatsauthor{ 
Dheeraj Baby$^{1}$\hspace{5mm}
Boran Han$^{1}$\hspace{5mm}
Shuai Zhang$^{1}$\hspace{5mm}
Cuixiong Hu$^{1}$\hspace{5mm}
Yuyang Wang$^{1}$\hspace{5mm}
Yu-Xiang Wang$^{1,2}$ 
}

\vspace{2mm}
\aistatsaddress{ 
$^{1}$Amazon \\
$^{2}$University of California San Diego \\
\vspace{1mm}
\small{Correspondence to: \texttt{dheerajbaby@gmail.com}}
}

\runningauthor{Dheeraj Baby, Boran Han, Shuai Zhang, Cuixiong Hu, Yuyang (Bernie) Wang, Yu-Xiang Wang}
]

\begin{abstract}
We study the well-motivated problem of \emph{online distribution shift} in which the data arrive in batches and the distribution of each batch can change arbitrarily over time. Since the shifts can be large or small, abrupt or gradual, the length of the relevant historical data to learn from may vary over time, which poses a major challenge in designing algorithms that can automatically adapt to the best ``attention span'' while remaining computationally efficient. We propose a meta-algorithm that takes any network architecture and any Online Learner (OL) algorithm as input and produces a new algorithm which \emph{provably} enhances the performance of the given OL under non-stationarity. Our algorithm is efficient (it requires maintaining only $O(\log(T))$ OL instances) and adaptive (it automatically chooses OL instances with the ideal ``attention'' length at every timestamp). Experiments on various real-world datasets across text and image modalities show that our method \emph{consistently} improves the accuracy of user specified OL algorithms for classification tasks. Key novel algorithmic ingredients include a \emph{multi-resolution instance} design inspired by wavelet theory and a \emph{cross-validation-through-time} technique. Both could be of independent interest.
\end{abstract}

\section{Introduction} \label{sec:intro}
Many real-world Machine Learning (ML) problems can be cast into the framework of online learning where a model continuously learns from an online datastream. For example, consider the task of classifying the gender from high-school yearbook images. Suppose that the data is presented in an online manner where at each timestamp, the learner is asked to classify images from that timestamp. 
An online ML model will continuously adjust its parameters based on the data it received sequentially. Shift in the data distribution across the timestamps in a datastream constitutes a significant challenge in the design of online learning algorithms.
For instance, in the case of high-school gender classification, the appearance characteristics of a population, such as fashion style or racial diversity, can evolve slowly over time. Such distribution shifts can cause models learned using old data to yield poor performance on the most recent or relevant data distribution. On the other hand, one can leverage an adaptively selected portion of the old data if the distribution is smoothly evolving. An ideal goal is to maximise the accuracy attainable at each timestamp (rather than maximising the average across all timestamps). Effectively handling this problem poses a common challenge in practical applications.

On the other hand, modeling how the data distribution evolves over time requires one to make restrictive assumptions while designing an effective learning algorithm. Unfortunately, most often the distribution shifts are caused by complex confounders which are hard to model \cite{zhu2014}. Consequently such assumptions may not be satisfied or even verifiable in practice. This leads to the phenomenon where the strong assumptions about the evolution of distribution can only contribute to more noise than signal into the process of algorithm design.

\begin{figure}
\vspace{-7mm}
\begin{minipage}{0.45\textwidth}
  \begin{framework}[H]
      \begin{algorithmic}[1]
      \STATE Initialize learner's hypothesis $h_1: \cX \rightarrow \mathcal Y$.
        \FOR{each round $t \in [T] := \{1,\ldots,T \}$}
                    
	\STATE Nature samples $n_t > 1$ covariate-label pairs  $(x_1,y_1),\ldots,(x_{n_t},y_{n_t})$ iid from a distribution $\mathcal{D}_t$ on the space $\cX \times \mathcal Y$. \\
        \STATE The covariates are revealed to the learner.
        \STATE Learner predicts their labels using $h_t$.
        \STATE True labels are revealed to the learner.
        \STATE Learner updates its hypothesis to $h_{t+1}$ using (a part of) the revealed labelled data.
        \ENDFOR
    \end{algorithmic}
    \caption{Batched online interaction protocol for classification}\label{fig:proto} 
\end{framework}
\end{minipage}
\vspace{-5mm}
\end{figure} 

To study this challenge closely, we consider the problem of online classification under distribution shifts without explicitly modeling the evolution of such shifts. A typical protocol for online classification is presented in Framework \ref{fig:proto}. In this paper, we use the term \textbf{online learner (OL)} to mean any learner that operates under the protocol in Framework \ref{fig:proto}. There is a rich body of work \cite{besbes2015non,zhang2018adaptive} (see Sec. \ref{sec:lit} for a broader overview) that studies principled ways of handling non-stationarity under convex loss functions. However, due to the huge success of deep learning, many of the modern ML systems use deep nets where the convexity assumptions are violated. This limits the applicability of methods that only handle convex losses to relatively simple use cases such as logistic regression, SVMs or fine-tuning the linear layer of a neural network.

Moving forward, our goal will be to effectively adapt to distribution shifts without imposing convexity assumptions on losses. For online learning problems, one can continually update the parameters of the underlying network based on the new data as it sequentially arrives. For example we can use online gradient descent or continual learning algorithms such as \cite{si} in hope to control the generalization error at each round in the online protocol. However, as noted in \cite{yao2022wild}, the performance of such methods can be limiting under distribution shifts. In this paper, we provide a meta-algorithm that takes in an arbitrary OL as a black-box and produces a new algorithm that has better classification accuracy under distribution shift. The black-box nature allows us to leverage the success of deep learning while still being able to adapt to distribution shift. This makes our method more widely applicable in practice in comparison to methods that only handle convex losses.  Our key contributions are as follows:


\setlist{nolistsep}
\begin{itemize}
\itemsep0em

    \item We develop a meta-algorithm \texttt{AWE} (Accuracy Weighted Ensemble, Alg.\ref{fig:awe}) that takes a black-box OL as input and improve its performance for online classification problem under distribution shifts. The method primarily contains two parts: (1) Multi-Resolution Instances (MRI) and (2) Cross-Validation-Through-Time (CVTT). 


    

    \item We obtain strong theoretical guarantees. For the MRI design, we show that it covers at least a fourth of all datapoints from most recent distribution (Theorem \ref{thm:mri}). We also give bounds for the generalization error and dynamic regret of \texttt{AWE}(Theorems \ref{thm:gen}, \ref{thm:dynamic}).

    \item We conduct experiments on various datasets with \emph{in-the-wild} distribution shifts across image and text modalities and find that our method \emph{consistently} leads to improved performance (Sec. \ref{sec:exp}) while incurring only a logarithmic overhead in memory and time in comparison to the base OL algorithm.
    
\end{itemize}

\textbf{Notes on key technical novelties.} 

\textbf{a) black-box nature:} We remark that the idea of enhancing the performance of a black-box algorithm under non-stationarity has been proposed in \cite{daniely2015strongly} (see Appendix \ref{app:daniely} for a comparison). The algorithm in \cite{daniely2015strongly} as well as ours are both based on running multiple instances of a base learner, where each instance is trained from a unique time-point in history and combining their predictions at a future time-point. However, algorithmic components that facilitate our black-box reduction differs from theirs in two aspects: i) \textbf{Data-efficient instances.} In Appendix \ref{app:daniely1}, by taking the specific case of a piece-wise stationary distribution shift, we show an example where the Geometric Covering intervals of \cite{daniely2015strongly} fails to guarantee existence of an instance that has been trained on adequate amount of data from most recent distribution. Our MRI construction of maintaining base learner instances fixes this problem (see Theorem \ref{thm:mri}) leading to a more data-efficient way of instantiating the base learners. ii) \textbf{Faster regret rates.} Suppose in Framework \ref{fig:proto}, after each round, $N$ labelled datapoints are revealed. The scheme in \cite{daniely2015strongly} guarantees an average regret of $O(1/\sqrt{|I|})$ against the best instance in any interval $I$. However, when the data distribution is slowly varying (or constant in the best case) within interval $I$, our scheme lead to a faster average regret of $O(1/\sqrt{N |I|})$. This is attributed to the CVTT technique which pools together datapoints from similar distribution in adjacent timestamps while estimating the loss of each instance in the recent distribution. The high accuracy estimates of losses allows us to quickly learn/identify the most appropriate instance for the recent distribution leading to fast regret rates. (see Appendix \ref{app:daniely2} for further details)


\textbf{b) logarithmic overhead:} Our MRI construction requires to maintain only a pool of $O(\log T)$ OL instances while guaranteeing the existence of instances in the pool that has been trained on reasonable amount of data exclusively from the recent distribution (see Theorem \ref{thm:mri}).

\textbf{c) comparison to } \cite{mazzetto2023adaptive}\textbf{:} 
We remark that a recent break-through due to  \cite{mazzetto2023adaptive} also provides a way for adapting to distribution shifts. However, their method involves solving \emph{multiple} ERM procedures at \emph{each} timestamp which is hard to deploy in online datastreams. In contrast, we introduce novel algorithmic components (see Sec. \ref{sec:refine}) that facilitate the adaptation of any \emph{user specified black-box} OL algorithms while obviating the need to solve multiple ERM procedures. The reason why  \cite{mazzetto2023adaptive} needs to perform multiple ERM procedures is that in their algorithm they need to compute the maximum mean discrepancy wrt a large hypothesis class. Our key observation is that (under piece-wise stationary distributions) we can compress a large hypothesis class to a set of finite learnt models with at-least one model being good for making predictions for the most recent distribution (see Theorem \ref{thm:mri}). Hence it suffices to compute the maximum mean discrepancy wrt this \emph{finite} set of models thereby leading to computational savings.

\section{Related Work} \label{sec:lit}
In this section, we briefly recall the works that are most related to our study. Adapting to distribution shifts under convex losses is well studied in literature \cite{hazan2007adaptive,zhang2018adaptive,zhang2018dynamic,Cutkosky2020ParameterfreeDA,baby2021optimal,Zhao2020DynamicRO,zhao2022eff,baby2022control,baby2023secondorder,baby2023nonstationary}. The strong requirement of convexity of losses limits their applicability to deep learning based solutions. Further, none of these works aim at optimizing Eq.\eqref{eq:err}. Methods in \cite{awasthi2023drift,jain2023instanceconditional} provide a non-black-box way to adapt to distribution shift in offline problems. There are also various online learning algorithms coming from rich body of literature involving continual learning and invariant risk minimization. Examples include but not limited to  \cite{si,ewc,Zhai2023,li2016,lee2017,aljundi2018,icarl2017,Chaudhry2019OnTE,Cai2021OnlineCL}. We refer the reader to  \cite{delange2021,yang2023} for a detailed literature survey. The aforementioned algorithms can be taken as the input OL for our methodology. Examples of schemes that tackle distribution shift under limited amount of labeled data include \cite{lipton2018blackbox,bai2022label,wu2021label,baby2023online,garg2023RLSBench,rosenfeld2023almost}. However, they require structural assumptions like label shift on the distribution shift owing to scarce labelled data available and hence form a complementary direction to our work.

\ul{We emphasize that our objective is not to attain the best classification accuracy across all algorithms. Instead, our goal is to propose an effective meta-algorithm that can enhance the accuracy of any given online algorithm for classification under distribution shifts.}

\section{Problem Setting}\label{sec:setting}

In this section, we define the notations used and metrics of interest that we aim to control. We use $[T] := \{1,\ldots,T \}$ and $[a,b] := \{a,\ldots,b \} \subseteq [T]$. Suppose we are at the beginning of round $t$. Let $\cX$ denote the covariate space and $\mathcal Y$ denote the label space. Suppose that the data distribution at round $t$ is $\cD_t$. Let $i$ denote an OL instance. Let ${\text{Acc}}_t(i) := E_{(x,y) \sim \cD_t}[\mathbb I \{ i(x) = y] \}$ where $i(x)$ is the prediction of the model $i$ for covariate $x$ and $\mathbb I$ is the binary indicator function. The accuracy ${\text{Acc}}_t(i)$ is the population level accuracy of model $i$ for the data distribution at round $t$. The black-box OL we take in as input to \texttt{AWE} will focus on updating the parameters of an underlying neural network architecture. Let $\cH$ be the hypothesis class defined by the underlying neural network classifier. Let $h_t^* = \argmax_{h \in \cH} {\text{Acc}}_t(h)$  be the best classifier for making prediction at round $t$. Suppose $h_t$ is the classifier used by our algorithm to make predictions for round $t$. One of the metrics we are interested in is controlling the instantaneous regret:
\begin{align}
    \text{Err}(t)
    &= {\text{Acc}}_t(h_t^*) - {\text{Acc}}_t(h_t), \label{eq:err}
\end{align}
for the maximum number of rounds possible. Doing so implies that the accuracy of our algorithm stays close to the best attainable performance by any classifier in $\cH$ across \emph{most} rounds. However, in rounds where the data distribution is very different from the past seen distributions, any algorithm will have to pay an unavoidable price proportional to the discrepancy between the distributions.

We also provide guarantees for the dynamic regret given by:
\begin{align}
    R_{\text{dynamic}} = \sum_{t=1}^T \text{Err}(t) 
    = \sum_{t=1}^T {\text{Acc}}_t(h_t^*) - {\text{Acc}}_t(h_t). \label{eq:dyn}
\end{align}

Controlling $\text{Err}(t)$ is more challenging than controlling the dynamic regret because the former can imply later. Bounding $\text{Err}(t)$ leads to a more stringent control of accuracy at each round in the datastream. This is one of the key formalizations that differentiates our setting from \cite{daniely2015strongly}. The control over instantaneous regret is translated to our experiments (see Sec. \ref{sec:exp}) by the improved accuracy of our meta-algorithm across \emph{most} of the timestamps in the data stream (see also Remark \ref{rmk:cmp}).

\section{Algorithm}

\begin{algorithm}[t]
        \caption{AWE: inputs: 1) A black-box online learning algorithm; 2) Failure probability $\delta$ and 3) Split probability $p$.} \label{fig:awe}
        \begin{flushleft}
	\begin{algorithmic}[1]
		\STATE 
            Initialize $w_1 := \mathbf{1} \in \mathbb{R}^{\log_2 T}$.
            \STATE Enumerate set of rules $\mathbbm{r} = \left \{ r^1,r^2,..,r^m \right \}$
            \FOR{ $t \in [T]$:}
            \STATE Get $n_t$ covariates $x_t(1),\ldots,x_t(n_t)$.
            \STATE Compute $\cA_{t} := \text{ACTIVE}(t)$ (See Eq.\eqref{eq:active}).
            \FOR{ $i \in n_t$:}
            \STATE Predict a category by giving $x_t(i)$ as input to \texttt{currentModel}
            \ENDFOR  

            \STATE Get labels $y_{1},\ldots,y_{n_t}$. Let $p$ fraction of the data be allocated to a training fold and remaining to a hold-out fold. 

            \STATE Train models in $\cA_t$ with the training fold using the online learning algorithm given as input.
            
            \STATE Compute $\cA_{t+1} := \text{ACTIVE}(t+1)$.
            \STATE For each model $i \in \cA_{t+1}$, compute the accuracy estimate $\widehat {\text{Acc}}_{t+1}^{(i)} = \texttt{refineAccuracy}(i,t,\delta)$ (see Alg.\ref{fig:refine}).
            \STATE Convert the accuracy values $\{\widehat {\text{Acc}}_{t+1}^{(i)}\}_{i=1}^{|\cA_{t+1}|}$ to weights $w_{t+1} \in \mathbb R^{|\cA_{t+1}|}$ (Eq.\eqref{eq:exp_weights}).            
            \STATE Construct the model $E_{t+1}: x \rightarrow \argmax_{k \in [K]} \sum_{i \in \cA_{t+1}}  w_{t+1}(i) \text{logit}_i[k]$, where $\text{logit}_i \in \mathbb R^K$ is the logits of the model $i \in \cA_{t+1}$ for a given input covariate.
            \STATE Let $i_{t+1}^* = \argmax_i \widehat {\text{Acc}}_{t+1}^{(i)}$. 
            
            \STATE if $\texttt{refineAccuracy}(E_{t+1},t,\delta) > \texttt{refineAccuracy}(i_{t+1}^*,t,\delta)$, then set \texttt{currentModel} as $E_{t+1}$ else set $\texttt{currentModel}$ as $i_{t+1}^*$.
            \ENDFOR  
	\end{algorithmic} 

        \end{flushleft}
         \vspace{-2mm}
\end{algorithm}

\begin{figure}
\vspace{-5mm}
    \begin{minipage}{0.47\textwidth}
    \begin{algorithm}[H]
        \caption{\texttt{refineAccuracy}: Inputs: 1) An OL instance $M$; 2) A terminal time $\tau$ and 3) Failure probability $\delta$.} \label{fig:refine}
        \begin{flushleft}
	\begin{algorithmic}[1]
		\STATE  Let $n(r)$ denote the number of hold out-data points accumulated in the interval $[\tau - r+1,\ldots,\tau]$.
        \STATE $S(r,\delta) := \sqrt{\log(T/\delta)/r} + \sqrt{20 \log T/r} $; Initialize $r = 1$.
            \WHILE{$r \le \tau/2$}
            \STATE Let $u_M(r)$ be the empirical accuracy for the model $M$ estimated using hold-out data in the rounds $[\tau - r+1,\ldots,\tau]$.
            \STATE If $|u_M(r) - u_M(2r)| \le 4S(n(r),\delta)$ then $r \leftarrow 2 \cdot r$; Else return $u(r)$
            \ENDWHILE
            \STATE Return $u_M(r)$
	\end{algorithmic} 
    \end{flushleft}
\end{algorithm}
\end{minipage}
\vspace{-8mm}
\end{figure}

In this section, we present our algorithm for handling distribution shifts and elaborate on the intuition behind the design principles. A formal treatment will be presented in the next section. Along the way we also explain the challenges faced and how they are overcome via algorithmic components. The full algorithm is presented in Alg.\ref{fig:awe}.

Throughout the design of the algorithm, we assume that distribution of data at round $t$ is same as the distribution of the data at time $t+1$ for most rounds. (More precisely, the number of times the distribution switches is assumed to grow sub-linearly over time). Such an assumption can be considered as relatively weak. On the other hand, if the number of points where the distribution switches grow linearly with time, it can be shown that learning is impossible in such a regime \cite{zhang2018adaptive}. We do not assume any prior knowledge/modelling assumptions on where the switches happen.

The input to \texttt{AWE} is a user-specified OL. We would like to improve the accuracy of the given OL under distribution shifts.  To motivate our techniques informally, consider the following thought experiment. Suppose in the interval $[1,t_0]$ the data is generated from a distribution $\cD_1$ while in the interval $[t_0,T]$ it is from a sufficiently different distribution $\cD_2$. If we know that a change in distribution has happened at time $t_0$, then we can start a new OL instance from time $t_0$. We remind the reader that when a new instance is started, its internal states will also be re-initialized. However absent such knowledge, one naive idea is to start a new instance at every past timestamp in the datastream. Then combine their predictions (based on validation accuracies) at a future timestamp. Unfortunately, such a scheme can be computationally expensive making it less attractive in practice. Further the question of how to combine the predictions from various instances in a statistically efficient manner also remains unclear.


It is challenging to maintain a pool of instances such that: 1) the growth rate of the pool's size is slow and 2) An instance that has been trained on adequate amount of data from the most recent distribution exists in the pool.  In the next section, we discuss our solution to address the above problems. The solution is based on carefully adapting the idea of Multi Resolution Analysis (MRA) from wavelet theory \cite{mallat1999wavelet}.

\subsection{Multi-resolution Instances (MRI)} \label{subsec:mri}
For the sake of simplicity, let  the time horizon $T$ be a power of $2$. We define $M:= \log_2 T$ resolutions. In each resolution $i \in [1,M]$ we define two collection of intervals as follows:
\begin{itemize}
\itemsep0em
    \item R:= $\left \{ [1+(k-1) T/2^{i-1}, k  T/2^{i-1} ] \text{ for } k \in \mathbb{N} \right \}$ 
    \item B:= $\Bigg. \{ [1+(k-1) T/2^{i-1} + T/2^{i}, k  T/2^{i-1} + 3\cdot T/2^{i}] 
                \text{ for } k \in \mathbb{N} \Bigg. \}$ ,
\end{itemize}
where we disregard the timepoints in an interval that exceeds the horizon $T$. We remark that intervals in the collection $B$ are not present in the MRA defined by usual wavelet theory. However, we include them to quickly pick up distribution shifts as formalised in Theorem \ref{thm:mri}. We refer the reader to Appendix \ref{app:proof} for a specific example demonstrating the motivation for including the set $B$. See Fig.\ref{fig:mra} for a depiction of the intervals defined by the MRI construction. 

With each interval, we associate an OL instance (hence interval and its associated instance are used interchangeably moving forward). For example, associated with an interval $[a,b]$ we define an OL instance that starts its operation at round $a$ and it is only used to make predictions within the interval $[a,b]$. For making a prediction at an intermediate time $t \in [a,b]$ this OL instance will only be trained on the data that is seen in the duration $[a,t-1]$. However, we remind the reader that prediction made by this instance at round $t$ may not be the final prediction submitted by the overall meta-algorithm. We remark that the instance defined by the interval $[a,b]$ is no longer used for making predictions at rounds $t > b$.

For any round $t$, define
\begin{align}
    \text{ACTIVE}(t) := \{ u \in R \cup B | t \in u\} \label{eq:active}.
\end{align}

This defines the collection of instances that are active at round $t$. The meta-algorithm will form a prediction only based on the active models. Due to construction of the intervals, it is straight-forward to see that at any round $t \in [T]$, we have $|\text{ACTIVE}(t)| = O(\log T)$.


In Theorem \ref{thm:mri}, we show that there exists a model in the MRI pool that has seen sufficient amount of data from a new distribution if a shift happens. Consequently we are able to maintain only logarithmic number of instances while still guaranteeing that there exists at least one instance in the pool that is efficient to make predictions for the most recent distribution. We emphasize that such a property is achieved without imposing strong modelling assumptions on the evolution of the shift.

\subsection{Cross-Validation-Through-Time (CVTT)} \label{sec:refine}
Now that we have a collection of instances, we next turn our attention to address the statistical challenge of combining the instances to make predictions that can lead to high accuracy. As mentioned in Framework \ref{fig:proto}, the labels for all covariates are revealed after each round $t$. We then split the data into a training fold and a hold-out fold (line 11 in Alg.\ref{fig:awe}). Data from the training fold is fed to the OL instances in $\text{ACTIVE}(t)$ to resume their training. 

\begin{figure}
    \centering
    \includegraphics[scale=0.4]{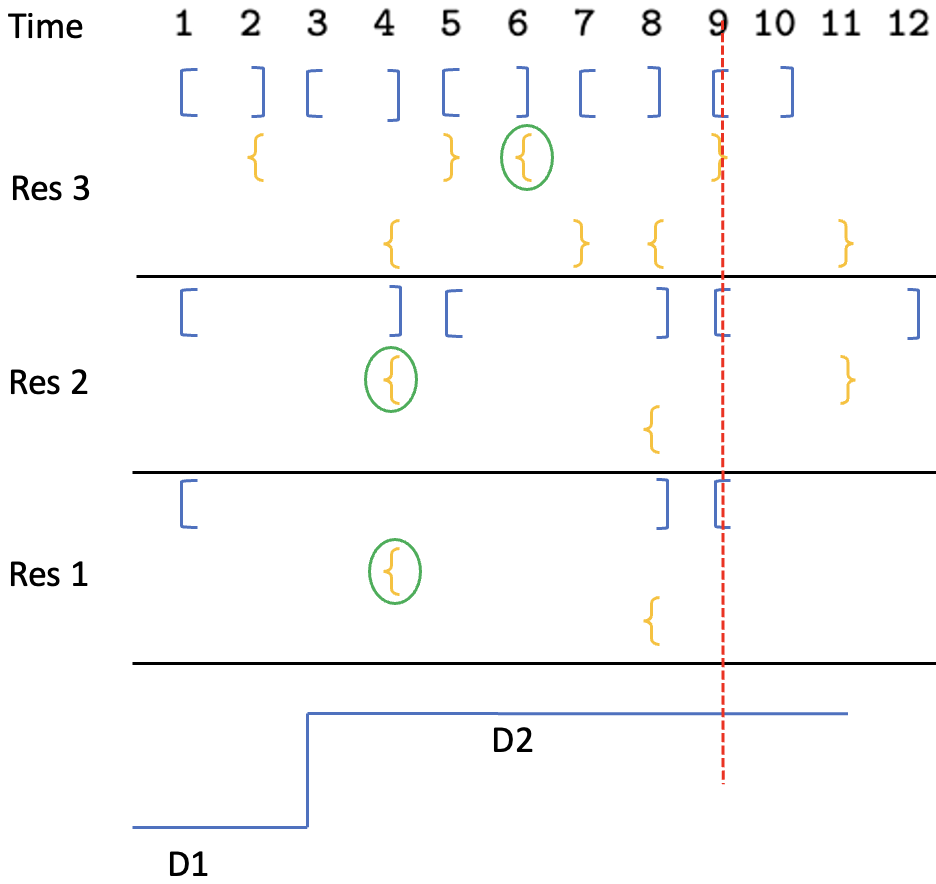}
    \caption{The figure shows the configuration of Multi Resolution Instances (MRI). Brackets of type $[$ $]$ belongs to the collection $R$ and type $\{ $ $ \}$ belongs to collection $B$ (see Sec. \ref{subsec:mri}). Consider the scenario where the data distribution has changed from timestamp $3$ and remained stable afterwards. Suppose we are at the beginning of round $9$ and after each round we get $n$ training data points. So we have seen $6n$ labelled data points from distribution $\cD_2$. $\text{ACTIVE}(9)$ corresponds to those intervals that include the timestamp $9$. The circled intervals has seen at least $3n$ data points from distribution $\cD_2$ thereby ensuring models that are present in the active set with good performance under distribution $\cD_2$. A formal result of the data utilization efficiency of the MRI construction is proved in Theorem \ref{thm:mri}.}
    \label{fig:mra}
    \vspace{-3mm}
\end{figure} 

At round $t+1$, we make predictions using the models in $\mathcal A_{t+1} := \text{ACTIVE}(t+1)$. Recall that throughout the design of our algorithm, we assume that the distribution of data at round $t$ is close to the distribution of the data at time $t+1$ for most rounds. Under such an assumption, the empirical accuracy for models in $\mathcal A_{t+1}$ computed using the hold-out data at round $t$ should be a rough estimate for the accuracy of those models for the data revealed at round $t+1$. Since we have only a limited amount of validation data, even if the distribution at time $t+1$ is identical to that at time $t$, such an estimate can be misleading. On the other-hand if the data distribution during the interval $[r,t]$ is relatively stable, then we can use all the hold-out data collected within $[r,t]$ to get a better estimate of the accuracy. To get such refined accuracy estimates (which in essence does a CVTT) for each model in $\mathcal A_{t+1}$, we use the recent advancement in \cite{mazzetto2023adaptive} adapted to our setting in Alg.\ref{fig:refine}. This idea of estimating refined accuracy for each instance that are present in a sparse pool of instances is what eliminates the need of using the techniques in \cite{mazzetto2023adaptive} that rely on expensive ERM computations. In contrast, techniques in \cite{mazzetto2023adaptive} require to compute a metric of the form $\sup_{M \in \cH} |u_M(r)- u_M(2r)|$ for a very large hypothesis class $\cH$ defined by the neural net architecture (cf. 
 Line 4 of Alg.\ref{fig:refine} for definition of $u_M(r)$). They use ERM for this purpose. However by virtue of our MRI construction and Theorem \ref{thm:mri}, we compress the large hypothesis class $\cH$ to a finite set of \emph{learnt} models (with at least one model that has seen at least a fourth of the datapoints from the current distribution). Hence it suffices only to compute such discrepancy metrics over a finite set formed by candidates for near optimal hypothesis for the current round. This facilitate the speedup over their methods. To the best of our knowledge, the idea of using techniques in \cite{mazzetto2023adaptive} to facilitate a CVTT has not been used before in literature.

Once the accuracy for each model in $\mathcal A_{t+1}$ is estimated, we form an ensemble model $E_{t+1}$ from those constituent models with weights specified by: 
\begin{align}
    w^{t+1}(i)  = \widehat {\text{Acc}}_{t+1}^{(i)}, \label{eq:exp_weights}
\end{align}
where $w^{t+1}_i$ is the weight assigned to model $i$ in round $t+1$ and $\widehat {\text{Acc}}_{t+1}^{(j)}$ is the estimated accuracy for model $j$. Then based on the refined accuracy estimate of all models (line 16), we pick a model to make predictions at line 7. This ensembling model is mainly introduced to get better performance in practice while not hurting the theoretical guarantees. It is perfectly possible to prefer any other ensembling scheme as well and still the guarantees of Theorem \ref{thm:gen} remains valid. The overall algorithm is displayed in Alg.\ref{fig:awe}.

\section{Theory} \label{sec:theory}
In this section, we present theoretical justifications of the algorithmic components of \texttt{AWE}. All proofs are deferred to Appendix \ref{app:proof}. For the sake of simplicity assume that we receive $n$ labelled training data points and $m$ hold-out data points after the end of each timestamp in Step 11 of Alg.\ref{fig:awe}.
Next theorem shows the ability of MRI to maintain instances in the pool that can lead to good predictions for the most recent distribution.

\begin{restatable}{theorem}{thmmri} \label{thm:mri}

 Suppose we are at the beginning of a timestamp $t+1$ and the data distribution has remained constant from some round $t_0 < t+1$. Let this distribution be $\cD$. We have labelled hold-out data available till round $t$. There exists at least one instance in the MRI pool that is active at a given round and satisfies at least one of the following properties:

\begin{itemize}
    \item All the training data seen from the model is from distribution $\cD$. Further it has seen at least $(t-t_0+1)n/2$ points from distribution $\cD$.


    \item The model been only trained on data from $\cD$. Further the number of points seen by the model is at least  $(t-t_0+1)n/4$.
    
\end{itemize}

\end{restatable}

Next, we attempt to understand the statistical efficiency of \texttt{AWE}. We define some notations first. Suppose the number of classes is $2$. Let $\mathcal F := \{ E_t\} \cup \cA_{t}$. Let $f_t^* = \argmax_{f \in \cF} {\text{Acc}}_t(f)$. $\hat f_t := \argmax_{f \in \cF} \widehat {\text{Acc}}_t(f)$. Let $h_t^* = \argmax_{h \in \cH} {\text{Acc}}_t(h)$.

\begin{restatable}{theorem}{thmgen} \label{thm:gen}
Assume the notations defined in Sec. \ref{sec:setting}. Suppose we are at the beginning of round $t$ and that the data is sampled independently across timestamps. Assume that the data distribution (say $\cD$) has remained constant in $[t-r,t]$. Then  with probability at least $1-4 \delta \log T \log mT$, instantaneous regret at round $t$ for \texttt{AWE},
\begin{align}
    {\text{Acc}}_t(h_t^*) - {\text{Acc}}_t(\hat f_t)
    &= {\text{Acc}}_t(h_t^*) - {\text{Acc}}_t(f_t^*) + \tilde O \left( \sqrt{1/mr }\right),
\end{align}
where $\tilde O$ hides logarithmic factors in $T,m,r$ and $1/\delta$.
\end{restatable}

\begin{remark} \label{rmk:rmk1}
Theorem \ref{thm:gen} must be interpreted in the light of Theorem \ref{thm:mri}. The first term ${\text{Acc}}_t(h_t^*) - {\text{Acc}}_t(g_t^*)$ depends on the generalization behaviour of the online learning algorithm given as input to \texttt{AWE}. However, due to Theorem \ref{thm:mri}, it is guaranteed that there exists at least one model in $\text{ACTIVE}(t)$ that has seen an $\Omega(mr)$ data points from the distribution $\cD$ when the data distribution has remained constant in $[t-r,t]$. This helps to keep the first term small. However, a theoretical bound on first term will depend on the specific OL algorithm used. For example, if we use ERM as the base learner, the first term can be bound by $O(1/\sqrt{nr})$. The second term  $\tilde O (\sqrt{1/mr})$ reflects improvement in generalization obtained by adaptively using all the hold-out data from previous $r$ rounds where the distribution has remained unchanged. We remark that the prior knowledge of $r$ (or the change-point) is not required.
\end{remark}

\begin{remark}\label{rmk:cmp}
Let $i_0$ be the OL instance started from timestamp $1$. By the construction of MRI in Sec. \ref{subsec:mri}, $i_0 \in \text{ACTIVE}(t)$ for all rounds $t$. Whenever $g_t^* = i_0$, the above theorem guarantees that the accuracy of \texttt{AWE} is comparable to that of $i_0$ at the asymptotically decaying margin of $O(1/\sqrt{mr})$. On the other hand, the distribution shifts can potentially cause the existence of new models $g_t^* \in \cG$ such that ${\text{Acc}}_t(g_t^*) \gg {\text{Acc}}_t(i_0)$. In such scenarios, Theorem \ref{thm:gen} again guarantees that the accuracy of \texttt{AWE} (with map $\hat \cW$) is comparable to that of ${\text{Acc}}_t(g_t^*)$ thereby improving the performance over the instance $i_0$. This solidifies the ability of \texttt{AWE} to enhance the performance of a user-given OL.
\end{remark}

The statement of Theorem \ref{thm:gen} translates into an excess risk bound against the pool of MRI instances. If we choose the input OL to \texttt{AWE} as ERM, one can get an excess risk bound of $O(1/\sqrt{mr})$ against an entire hypothesis class across which ERM is performed. However in an online setting, running ERM at each round can be computationally prohibitive. Hence the alternate choices for the OL (for example continual learning algorithms) becomes more useful. In such a scenario, Theorem \ref{thm:gen} guarantees an excess risk bound against an optimally restarted instantiation of the chosen OL -- thanks to the data-efficiency guarantees of MRI construction (Theorem \ref{thm:mri}).

Even-though we presented our analysis for piecewise stationary distributions, the extension to slowly varying distributions is straightforward by standard discretization arguments. In the analysis we can divide the time horizon into intervals where distribution is slowly varying and apply Theorem \ref{thm:gen} to the mean distribution within the bin while paying a small additive price proportional to how much the distribution at a round deviates from the mean. In Appendix \ref{app:proof}, we characterize the dynamic regret (see Eq.\eqref{eq:dyn}) of \texttt{AWE} on slowly changing distribution shifts. Though \texttt{AWE} has the drawback of storing all the previous hold-out data points, we remark that size of hold-out data is generally much smaller than the training data set size. In practice, one can limit to store only the hold-out data from a fixed amount of past timestamps. While refined accuracy calculation necessitates a forward inference pass on this data, the efficiency of modern GPUs prevents significant performance issues.

\section{Experiments}\label{sec:exp} 
\begin{figure*}[h]
    \centering
    \subfloat[FMOW]{\includegraphics[width=0.25\linewidth]{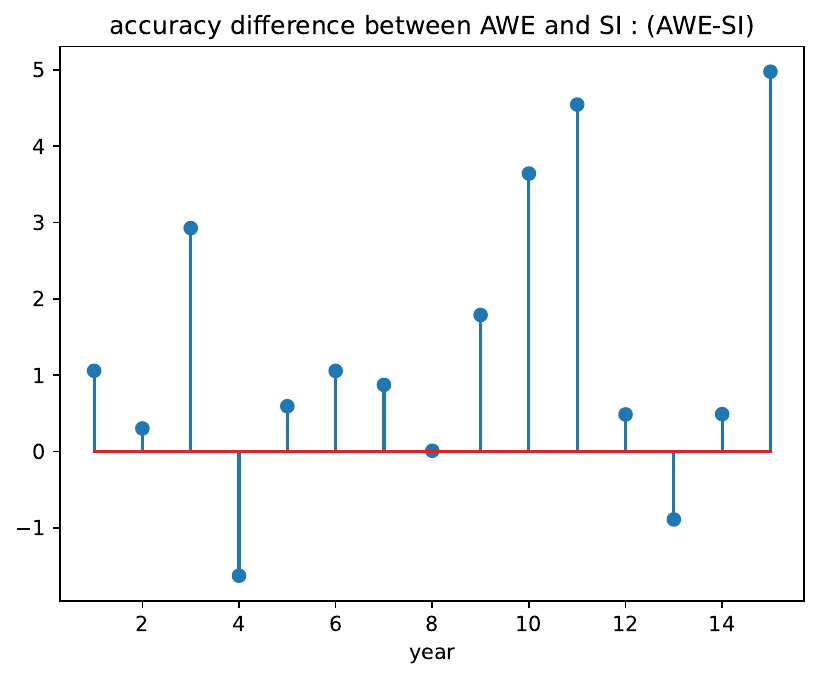}}\hfill
    \subfloat[Huffpost]{\includegraphics[width=0.25\linewidth]{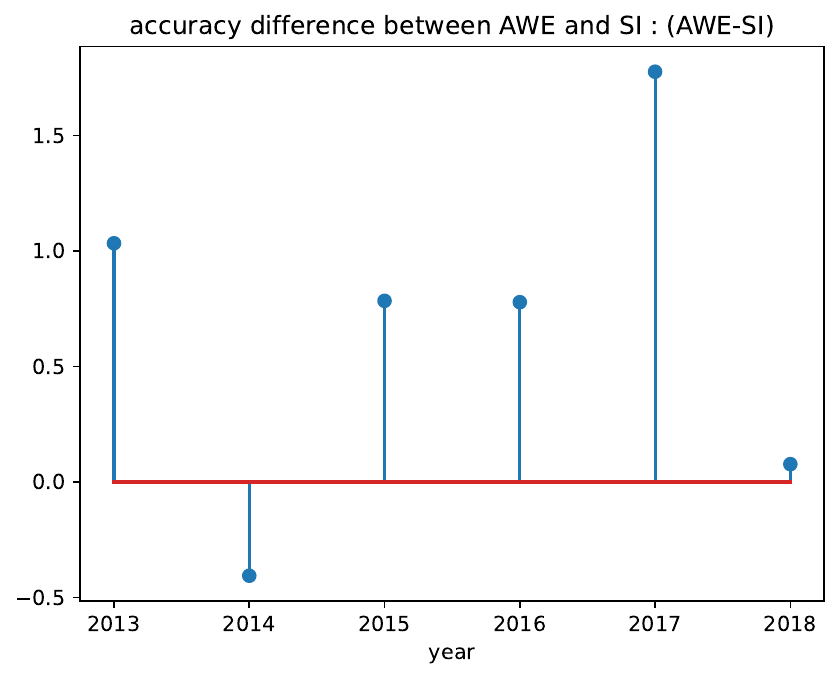}}\hfill
    \subfloat[Arxiv]{\includegraphics[width=0.25\linewidth]{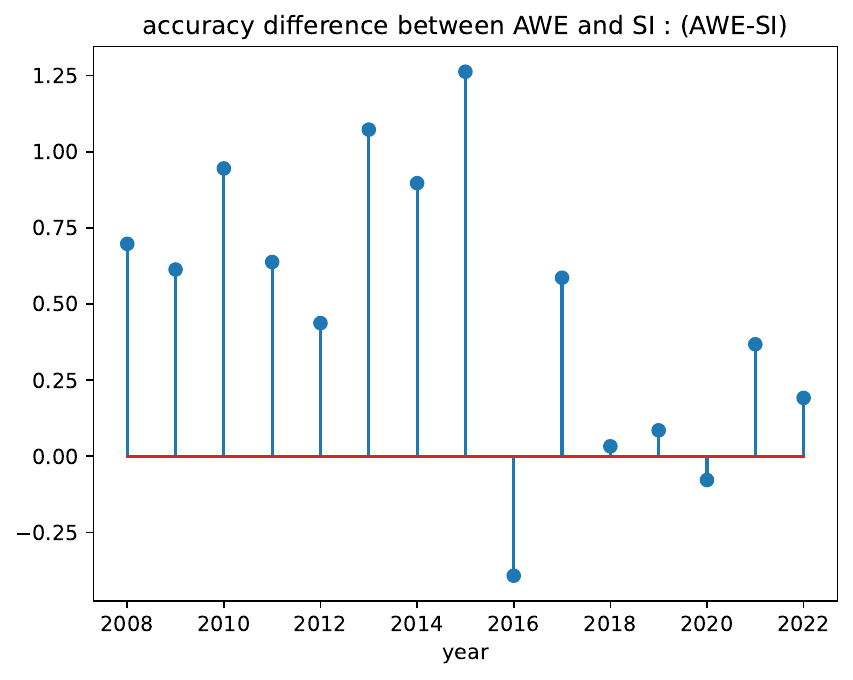}}

    \caption{$\%$ accuracy differences across various timestamps when \texttt{AWE} is run with SI as the online learning algorithm. We report similar results for other OL algorithms and the fraction of timestamps where \texttt{AWE} improves (or does not degrade) the performance of the base OL in Appendix \ref{app:exp}.}
    \label{fig:images_si}
    \vspace{-2mm}
\end{figure*}

\setlength{\tabcolsep}{8.5pt}
\begin{table*}[h] 
\tiny
\centering
\caption{Performance statistics for image (FMOW dataset) and text (Huffpost \& Arxiv datasets) modalities. We report the difference in average classification accuracy ($\%$) across all timestamps obtained by a black-box scheme minus that of the input OL. }
\begin{tabular}{c|cc|cc|cc}
\specialrule{.2em}{.1em}{.1em}
\hline
\multirow{2}{*}{\begin{tabular}[c]{@{}c@{}}\\ Input OL\end{tabular}} 
& \multicolumn{2}{c|}{FMOW} & \multicolumn{2}{c|}{Huffpost} & \multicolumn{2}{c}{ArXiv} \\ \cline{2-7} 
& \multicolumn{1}{c|}{SAOL \cite{daniely2015strongly}} & AWE  
& \multicolumn{1}{c|}{SAOL \cite{daniely2015strongly}} & AWE  
& \multicolumn{1}{c|}{SAOL \cite{daniely2015strongly}} & AWE  \\ \specialrule{.2em}{.1em}{.1em}

SI \cite{si}    
& \multicolumn{1}{c|}{\begin{tabular}[c]{@{}c@{}}$-4.19$\\ $ \pm 0.119$\end{tabular}} 
& \begin{tabular}[c]{@{}c@{}}$ \mathbf{1.52}$\\ $\mathbf{ \pm 0.067}$\end{tabular} 
& \multicolumn{1}{c|}{\begin{tabular}[c]{@{}c@{}}$-0.70$\\ $ \pm 0.068$\end{tabular}} 
& \begin{tabular}[c]{@{}c@{}}$\mathbf{0.70}$\\ $\mathbf{\pm 0.065}$\end{tabular} 
& \multicolumn{1}{c|}{\begin{tabular}[c]{@{}c@{}}$-3.97$\\ $\pm 0.159$\end{tabular}} 
& \begin{tabular}[c]{@{}c@{}}$\mathbf{0.45}$\\ $\mathbf{\pm 0.009}$\end{tabular} \\ \hline

FT  
& \multicolumn{1}{c|}{\begin{tabular}[c]{@{}c@{}}$-3.71$\\ $ \pm 0.104$\end{tabular}} 
& \begin{tabular}[c]{@{}c@{}}$\mathbf{1.83}$ \\ $\mathbf{\pm 0.073}$\end{tabular}  
& \multicolumn{1}{c|}{\begin{tabular}[c]{@{}c@{}}$0.71$\\ $\pm 0.06$\end{tabular}}    
& \begin{tabular}[c]{@{}c@{}}$\mathbf{0.72}$\\ $\mathbf{\pm 0.069}$\end{tabular} 
& \multicolumn{1}{c|}{\begin{tabular}[c]{@{}c@{}}$-3.95$\\ $\pm 0.14$\end{tabular}}  
& \begin{tabular}[c]{@{}c@{}}$\mathbf{0.56}$\\ $\mathbf{\pm 0.01}$\end{tabular}  \\ \hline

IRM \cite{irm}  
& \multicolumn{1}{c|}{\begin{tabular}[c]{@{}c@{}}$-6.16$\\ $ \pm 0.132$\end{tabular}} 
& \begin{tabular}[c]{@{}c@{}}$\mathbf{0.55}$\\ $\mathbf{ \pm 0.04}$\end{tabular}   
& \multicolumn{1}{c|}{\begin{tabular}[c]{@{}c@{}}$0.37$\\ $\pm 0.049$\end{tabular}}   
& \begin{tabular}[c]{@{}c@{}}$\mathbf{0.98}$\\ $\mathbf{\pm 0.08}$\end{tabular}  
& \multicolumn{1}{c|}{\begin{tabular}[c]{@{}c@{}}$-2.67$\\ $\pm 0.131$\end{tabular}} 
& \begin{tabular}[c]{@{}c@{}}$\mathbf{0.13}$\\ $\mathbf{\pm 0.005}$\end{tabular} \\ \hline

EWC \cite{ewc}  
& \multicolumn{1}{c|}{\begin{tabular}[c]{@{}c@{}}$-3.50$\\ $ \pm 0.101$\end{tabular}} 
& \begin{tabular}[c]{@{}c@{}}$\mathbf{2.20}$\\ $\mathbf{ \pm 0.08}$\end{tabular}   
& \multicolumn{1}{c|}{\begin{tabular}[c]{@{}c@{}}$-0.53$\\ $\pm 0.059$\end{tabular}}  
& \begin{tabular}[c]{@{}c@{}}$\mathbf{0.72}$\\ $\mathbf{\pm 0.069}$\end{tabular} 
& \multicolumn{1}{c|}{\begin{tabular}[c]{@{}c@{}}$-3.87$\\ $\pm 0.157$\end{tabular}} 
& \begin{tabular}[c]{@{}c@{}}$\mathbf{0.36}$\\ $\mathbf{\pm 0.008}$\end{tabular} \\ \hline

CORAL \cite{coral}  
& \multicolumn{1}{c|}{\begin{tabular}[c]{@{}c@{}}$-2.98$\\ $ \pm 0.093$\end{tabular}} 
& \begin{tabular}[c]{@{}c@{}}$\mathbf{3.24}$\\ $\mathbf{ \pm 0.097}$\end{tabular}  
& \multicolumn{1}{c|}{\begin{tabular}[c]{@{}c@{}}$0.18$\\ $\pm 0.034$\end{tabular}}   
& \begin{tabular}[c]{@{}c@{}}$\mathbf{1.10}$\\ $\mathbf{\pm 0.085}$\end{tabular} 
& \multicolumn{1}{c|}{\begin{tabular}[c]{@{}c@{}}$-1.16$\\ $\pm 0.087$\end{tabular}} 
& \begin{tabular}[c]{@{}c@{}}$\mathbf{1.67}$\\ $\mathbf{\pm 0.018}$\end{tabular} \\ \hline 

\specialrule{.2em}{.1em}{.1em}
\end{tabular}
\label{tab:acc}
\vspace{-3mm}
\end{table*}

\begin{figure*}[h]
    \centering
    \subfloat[FMOW]{\includegraphics[width=0.25\linewidth]{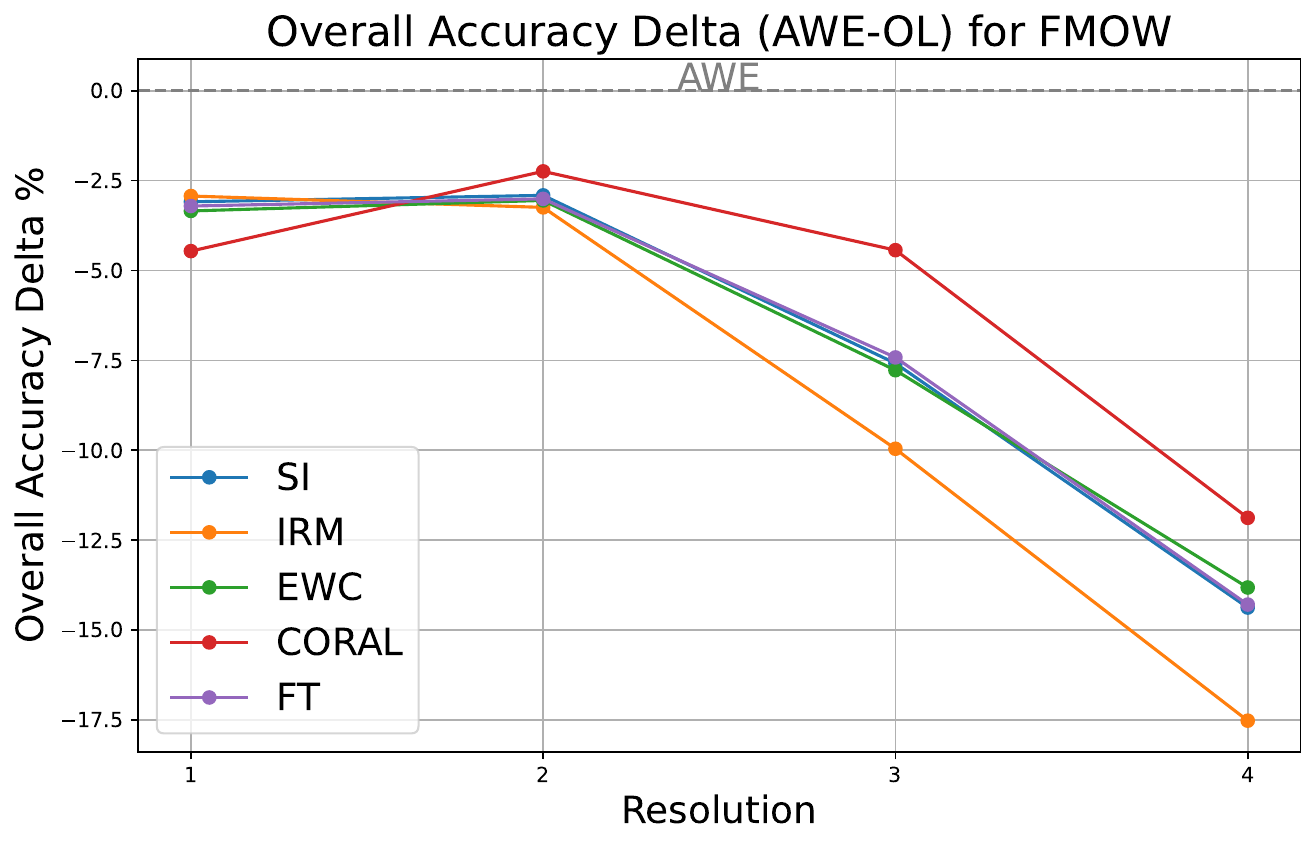}}\hfill
    \subfloat[Huffpost]{\includegraphics[width=0.25\linewidth]{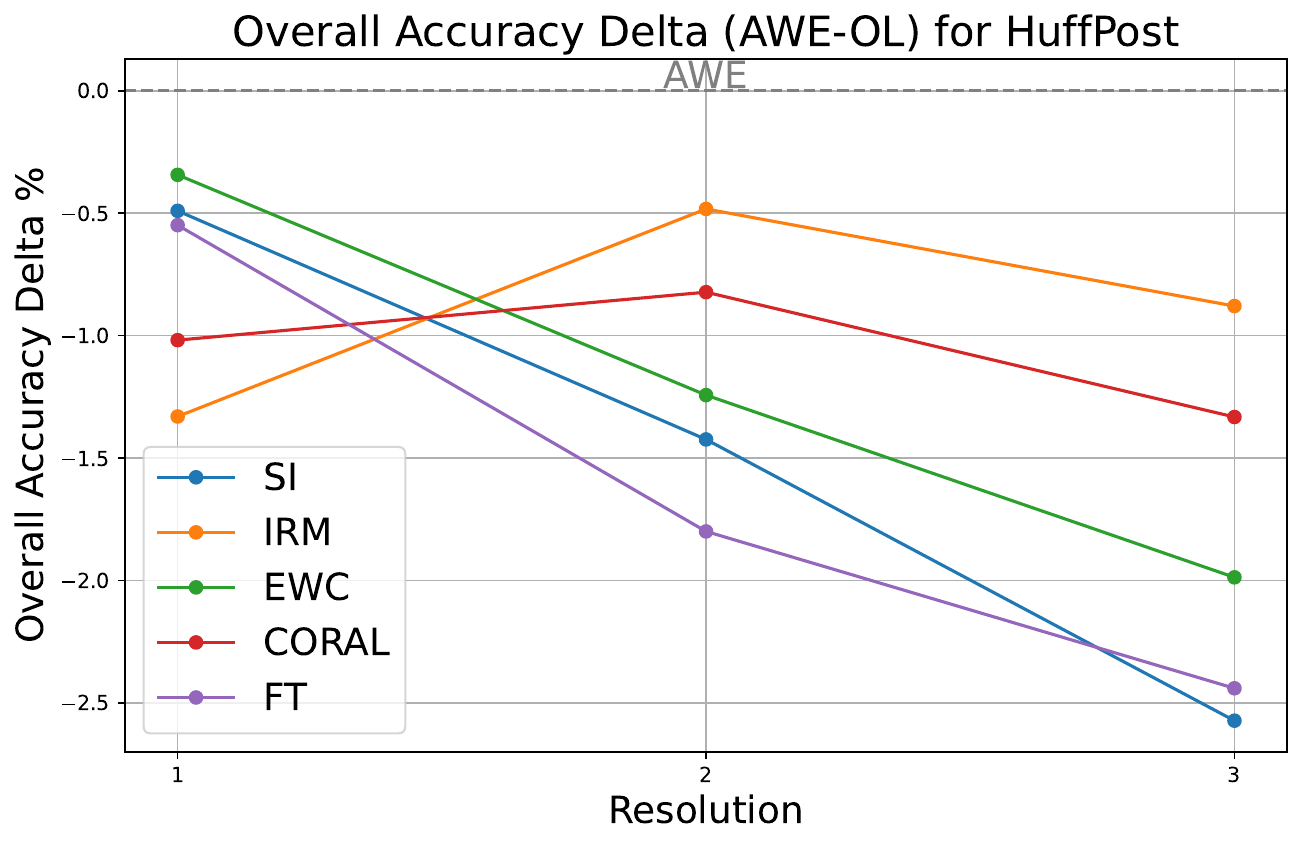}}\hfill
    \subfloat[Arxiv]{\includegraphics[width=0.25\linewidth]{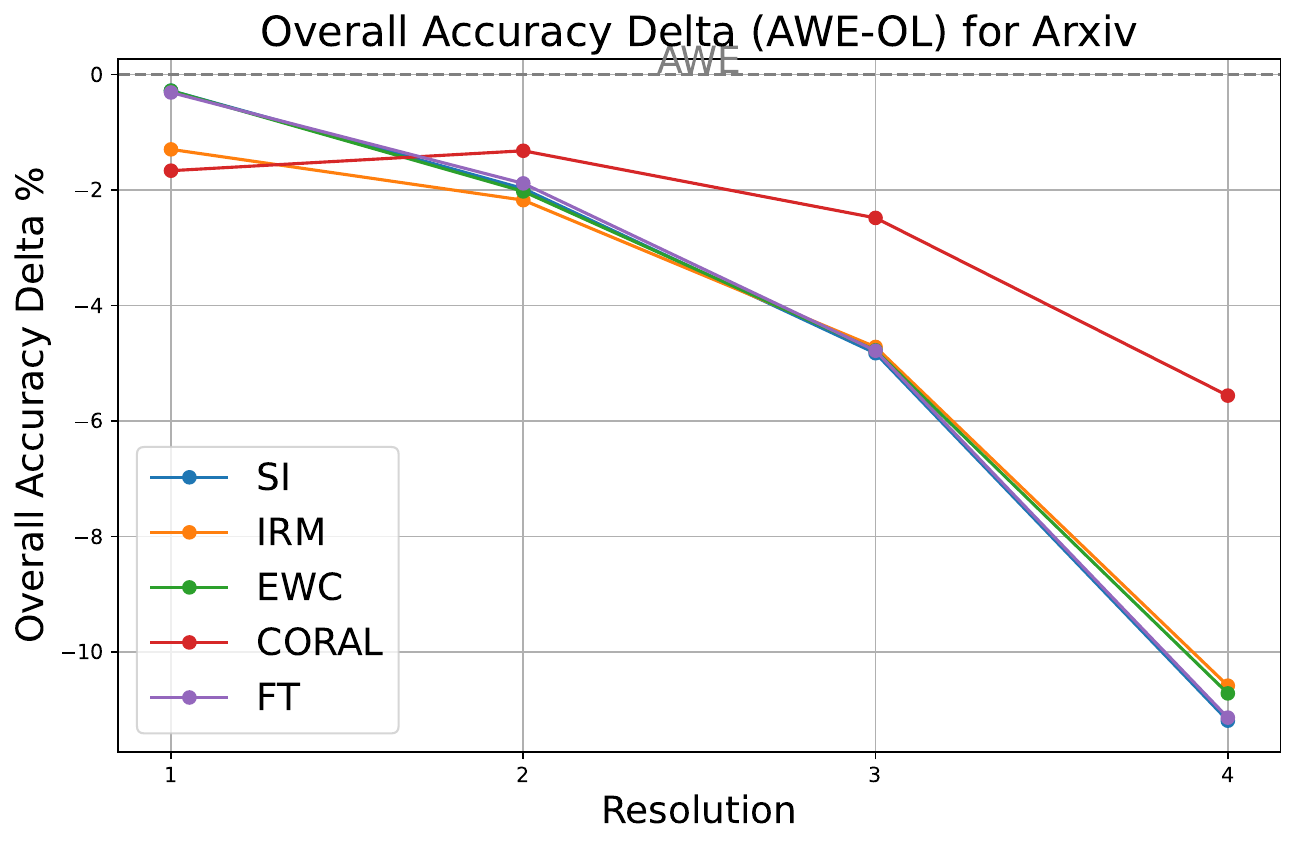}}

    \caption{Ablation study across various resolutions. We compute the overall accuracy attained by \texttt{AWE} minus that attained by using only a single resolution in the MRI pool. We see that in most cases \texttt{AWE} outperforms the single resolution counterparts. Further, by virtue of using \texttt{AWE}, the user does not need to hand-tune the optimal resolution to use in an MRI pool.}
    \label{fig:images_mri}
    \vspace{-3mm}
\end{figure*}

\subsection{Empirical Study on Real-world Datasets}

All our experiments are conducted on the WildTime benchmark \cite{yao2022wild}. WildTime constitutes a suite of datasets for classification problems across image, text and tabular modalities that exhibit natural distribution shift. In this paper, our focus is on the image and text modalities. For the image modality, the WildTime benchmark comprises the FMOW dataset, while for text, it includes the Huffpost and Arxiv datasets. We direct the reader to \cite{yao2022wild} for more elaborate details regarding the datasets. All the experiments were conducted on NVIDIA A100 GPUs.


For each of the dataset, we consider different black-box OL algorithms as input to \texttt{AWE} (Alg. \ref{fig:awe}). We then compare the accuracy of each OL and \texttt{AWE} method across every time stamp. To clarify further, at any timestamp, we compute the accuracy only using the labelled data that is revealed towards the end of that timestamp (see Framework \ref{fig:proto}). We explore the performance of five different online learning algorithms: Synaptic Intelligence (SI) \cite{si}, Invariant Risk Minimisation (IRM) \cite{irm}, FineTuning (FT), Elastic Weight Consolidation (EWC) \cite{ewc}, and CORAL \cite{coral}. We use the same neural network architectures and hyper-parameter choices for the OL as used by the eval-stream setting in \cite{yao2022wild}. We refer the reader to \cite{yao2022wild} for comprehensive information about these design choices. 

The performance compared with SI along each timestamp is shown in Fig.\ref{fig:images_si} for each dataset. Similar results for other OL algorithms are presented in Appendix \ref{app:exp}. We combine the predictions of a model at round $t$ using the accuracy estimates based on the data revealed until time $t-1$. Occasional drops in performance at certain timestamps may result from abrupt changes in the distribution. If the distribution at time $t$ is sufficiently different from that at time $t-1$, the accuracy estimates we use to combine the instances can be not useful for making predictions at round $t$. However, our algorithm demonstrates a rapid adjustment to the new distribution, as evidenced by performance improvements shortly after the timestamps where a performance decline was observed.

To statistically summarize the performance, we report the $\%$ accuracy differences (i.e., $\texttt{AWE} - \text{OL}$) across various time stamps. in Table \ref{tab:acc}. Additionally, in Appendix \ref{app:exp}, we detail the win/draw/lose numbers along the timestamps. 
We find that in well above 50\% of the timestamps across all cases, AWE can improve the performance of the base OL instance. This indicates the robustness of \texttt{AWE} to deliver improved performance under \emph{in-the-wild distribution shifts}.

We also compare \texttt{AWE} with Strongly Adaptive Online Learning (SAOL) \cite{daniely2015strongly}. To the best of our knowledge, SAOL is the only black-box adaptation scheme applicable to online non-convex setting. All runs are repeated with $3$ random seeds. We compute the difference in average accuracy (i.e \texttt{AWE} - OL or SAOL - OL) over all timestamps attained by both schemes in Table \ref{tab:acc}.
In contrast, we find that SAOL often degrades the performance of the input OL which can be alluded to the reasons presented in part (a) of notes on technical novelties in Sec. \ref{sec:intro}. 


\subsection{Ablation study} \label{sec:abl}

\textbf{Ablation on resolution:} To further substantiate the efficacy of the multi-resolution instance paradigm (Sec. \ref{subsec:mri}), we compare the performance obtained by \texttt{AWE} and an alternative where we only maintain instances from a single resolution in the MRI construction. Using a single resolution has the effect of dividing the entire time horizon into fixed-size intervals along the sets $R$ and $B$ (that are potentially overlapping; Sec. \ref{subsec:mri}). Each window is associated with an OL instance trained exclusively on the duration of that interval. We see in Fig.\ref{fig:images_mri} that in most of the cases \texttt{AWE} attains superior performance. Choosing an optimal window size (and hence the resolution number) requires prior knowledge of the type of shift. In contrast, \texttt{AWE} eliminates the need for this window-size selection by adaptively combining instances from multiple resolutions, automatically assigning a higher weight to the higher-performing resolutions during prediction. We refer the reader to Appendix \ref{app:exp} for more experimental results.

\section{Conclusion and Future Work} \label{sec:concl}
We proposed a method to enhance the performance of any user given OL for classification by merging predictions from different historical OL instances. The MRI construction for limiting the instance-pool size could be of independent interest for tackling non-stationarity. Experiments across various real world datasets indicate that our method \emph{consistently} improves the performance of user-specified OL algorithms. Future work include extending the methods to change point detection and online learning with limited feedback. Limitations can be found at Appendix \ref{limitation}.

\bibliographystyle{apalike}
\bibliography{nuerips/tf,nuerips/yx}

\section*{Checklist}

\textbf{In your paper, please delete this instructions block and only keep the Checklist section heading above along with the questions/answers below.}

 \begin{enumerate}

 \item For all models and algorithms presented, check if you include:
 \begin{enumerate}
   \item A clear description of the mathematical setting, assumptions, algorithm, and/or model. [Yes]
   \item An analysis of the properties and complexity (time, space, sample size) of any algorithm. [Yes]
   \item (Optional) Anonymized source code, with specification of all dependencies, including external libraries. [Yes. Please check the supplementary zip file.]
 \end{enumerate}

 \item For any theoretical claim, check if you include:
 \begin{enumerate}
   \item Statements of the full set of assumptions of all theoretical results. [Yes]
   \item Complete proofs of all theoretical results. [Yes]
   \item Clear explanations of any assumptions. [Yes]     
 \end{enumerate}

 \item For all figures and tables that present empirical results, check if you include:
 \begin{enumerate}
   \item The code, data, and instructions needed to reproduce the main experimental results (either in the supplemental material or as a URL). [Yes]
   \item All the training details (e.g., data splits, hyperparameters, how they were chosen). [Yes]
         \item A clear definition of the specific measure or statistics and error bars (e.g., with respect to the random seed after running experiments multiple times). [Yes]
         \item A description of the computing infrastructure used. (e.g., type of GPUs, internal cluster, or cloud provider). [Yes]
 \end{enumerate}

 \item If you are using existing assets (e.g., code, data, models) or curating/releasing new assets, check if you include:
 \begin{enumerate}
   \item Citations of the creator If your work uses existing assets. [Yes]
   \item The license information of the assets, if applicable. [Yes]
   \item New assets either in the supplemental material or as a URL, if applicable. [Yes]
   \item Information about consent from data providers/curators. [Yes]
   \item Discussion of sensible content if applicable, e.g., personally identifiable information or offensive content. [Not Applicable]
 \end{enumerate}

 \item If you used crowdsourcing or conducted research with human subjects, check if you include:
 \begin{enumerate}
   \item The full text of instructions given to participants and screenshots. [Not Applicable]
   \item Descriptions of potential participant risks, with links to Institutional Review Board (IRB) approvals if applicable. [Not Applicable]
   \item The estimated hourly wage paid to participants and the total amount spent on participant compensation. [Not Applicable]
 \end{enumerate}

 \end{enumerate}

\newpage
\appendix
\onecolumn

\section{Limitations and Further Discussion} \label{limitation}
\textbf{Limitaions.} One of the main limitations of our work is that we only try to optimize the classification accuracy at any round. However for the case of continual learning based algorithms, other metrics which measure the power of remembering old experiences to prevent catastrophic forgetting is also often optimized. Such metrics are often termed as backward transfer metrics. Our method do not explicitly optimize metrics that measure backward transfer, though it could be internally optimized by the different instances we maintain in the ensemble pool. Further, our algorithms do not take into account protecting the privacy of sensitive training data. 

\textbf{Further discussion.} Adapting to distribution shifts is a well studied problem in online learning literature \cite{zhang2018adaptive,Zhao2020DynamicRO,Baby2021OptimalDR}. However, these works are studied under convexity assumptions on the loss functions. To the best of of our knowledge adaptive regret minimization works like \cite{daniely2015strongly} are the only ones that are also applicable to the deep learning framework. Eventhough they are backed by strong theoretical guarantees, their practical performance is curiously under-investigated in literature before. Our experiments in Section \ref{sec:exp} indicate a large theory-practice gap for adaptive algorithm such as \cite{daniely2015strongly} when applied to deep learning framework. Our work is a first step towards closing this gap.

\section{Omitted Technical Details} \label{app:proof}

\textbf{An example motivating the necessity of including the B set in the MRI construction.} We included the $B$ set because it turned out to be important for proving Theorem 3. 

Fig.\ref{fig:mra} highlights a distribution shift scenario where usage of B set turns out to be useful. In the figure we are at the beginning of round $9$ and we have seen $6n$ labelled datapoints from distribution $D_2$ as explained in the caption of the figure. We can see that all the active blue brackets starts only from time $9$ and hence has not seen any data from distribution $D_2$. But the circled brackets from set B in resolutions $1$ and $2$ has seen $5n$ data points from $D_2$. This forms a particular scenario where the brackets from $B$ turns out to be useful in getting a data coverage guarantee as stipulated by Theorem \ref{thm:mri}

\thmmri*
\begin{proof}
Through out this proof we view intervals in $R \cup B$ and models in MRI synonymously since they are strongly associated with each other. We assume that time $t_0$ is not the start of any intervals in the MRI, otherwise the theorem is trivially satisfied.

Since we are the beginning of round $t+1$, we have labelled training data available for rounds in $[t_0,t]$. Consider the smallest resolution in the MRI where there exists an interval $[a,b]$ in $R$ that includes the duration $[t_0,t+1]$. Let the length of such an interval be $d$. For brevity of notations let's define $t_0' := t_0-a+1$ and $t' = t-a+1$. Since this is the interval in the smallest resolution that covers $[t_0,t+1]$, we must have that $t_0' \le d/2$.

Recall that the data distribution $\cD$ is constant  across rounds $[t_0,t+1]$. Note that $d > t'$. So if $t'-d/2 \ge d/2-t'+1$, then we can select the interval $[(a+b)/2+1,b] \in R$ that has seen at-least $(t-t_0+1)n/2$ from distribution $\cD$. Since the current timepoint $t+1 \in [a,b]$, we also have that $[(a+b)/2+1,b] \in \text{ACTIVE}(t+1)$

In the case when $t_0' \le d/4$, the interval $[a,b]$ has seen all the training data in $[t_0,t]$. Notice that the interval $[1+(3a+b)/4,(3a+b)/4+d] \in B$, covers at-least $(t-t_0+1)n/2$ of data points from the distribution $\cD$. Further this interval has seen only data from the distribution $\cD$. Since the duration of this interval is $d$, we conclude that is is active at the current round.


Now we consider the case where $d/4 < t_0' < d/2$ and $t'-d/2 < d/2-t'+1$. Consider the smallest resolution in MRI which contains a bracket $[e,f] \in R$ that fully covers the interval $[t_0,(a+b)/2]$. Clearly we must have $f = (a+b)/2$. Due to the smallest resolution property, we have that $t_0$ must be at-most $(e+f)/2$.

Since $t'-d/2 < d/2-t'+1$, we have that the number of data points seen within the interval $[(e+f)/2+1,f]$ must be at-least $(t-t_0+1)n/4$. Consequently the interval $[(e+f)/2+1,(5f-3e)/2] \in B$ must also have seen at-least $(t-t_0+1)n/4$ training data points from the distribution $\cD$. Since $t'-d/2 < d/2-t'+1$, we must have $[(e+f)/2+1,(5f-3e)/2] \in \text{ACTIVE}(t+1)$. Further since $\cD$ has remained constant across $[t+0,t+1]$, we have that all the data seen by the interval $(e+f)/2+1,(5f-3e)/2$ till now must be from $\cD$.  

\end{proof}

\begin{definition} \label{def:shatter}
Let $\cF$ be a function that maps $\mathcal Z$ to ${0,1}$. The shattering coefficient is defined as the maximum number of behaviours over $n$ points.
\begin{align}
    S(\cF,n) := \max_{z_{1:n} \in \mathcal Z} |\{(f(z_1),\ldots,f(z_n)) : f \in \cF\} |.
\end{align}
We say that subset $\cF' \subseteq \cF$ is an $n$-shattering-set if it is a smallest subset of $\cF$ such that for any $(f(z_1),\ldots,f(z_n)$ there exists some $f' \in \cF'$ such that  $(f(z_1),\ldots,f(z_n) =  (f'(z_1),\ldots,f'(z_n)$.

\end{definition}

\textbf{Notations}: We introduce some notations.  Let $d = |\text{ACTIVE}(t)|$, $\text{ACTIVE}(t) := \{M_1,\ldots,M_d \}$ and $\cG := \{ x \rightarrow \argmax_{k \in [K]} \left(\sum_{j=1}^{d} w_j \text{logit}(M_j(x))[k] \right) : w_j \in \mathbb{R} \}$ denote a hypothesis class consisting of classifiers whose logits are weighted linear combination of that in $\text{ACTIVE}(t)$. Let the distribution of data at round $t$ be $D_t$. Define $\text{Acc}_t(g) = E_{(X,Y) \sim D_t} [I\{ g(X) = Y\}]$ and $\widehat {\text{Acc}}_t(g) := \texttt{refineAccuracy}(g,t,\delta)$ for a classifier $g$.

\thmgen*
\begin{proof}

We decompose the instantaneous regret at round $t$ as

\begin{align}
    {\text{Acc}}_t(h^*) - {\text{Acc}}_t(\hat i)
    &=\underbrace{{\text{Acc}}_t(h^*) - {\text{Acc}}_t(f^*)}_{T_1} + \underbrace{{\text{Acc}}_t(f^*) - {\text{Acc}}_t(\hat f)}_{T_2}    \label{eq:decomp1}
\end{align}

We further proceed to bound $T_2$ as

\begin{align}
    T_2
    &= {\text{Acc}}_t(f^*) - \widehat {\text{Acc}}_t(f^*) + \widehat {\text{Acc}}_t(\hat f) - {\text{Acc}}_t(\hat f)\\
    &\quad \widehat {\text{Acc}}_t(f^*) - \widehat {\text{Acc}}_t(\hat f)\\
    &\le {\text{Acc}}_t(f^*) - \widehat {\text{Acc}}_t(f^*) + \widehat {\text{Acc}}_t(\hat f) - {\text{Acc}}_t(\hat f) \label{eq:decomp},
\end{align}
where the last line is due to the fact that $\hat f$ maximises the refined accuracy estimates among $\cF$.

Next, we proceed to bound $ |\widehat {\text{Acc}}_t(g) - {\text{Acc}}_t(g)|$ for any $g \in \cG$. Note that $\cF \subset \cG$. Hence such a task will directly lead to a bound on terms of the form $ |\widehat {\text{Acc}}_t(f) - {\text{Acc}}_t(f)|$ for any $f \in \cF$. The reason we follow this path is because the refined accuracy of the model $E \in \cF$ depends highly non-linearly on the past cross-validation data due to the weighted combination of models in $\text{ACTIVE}_t$ where the weights itself are based on the corresponding models' refined accuracy estimate.

Subtle care needs to be exercised when bounding terms of the form $ |\widehat {\text{Acc}}_t(g) - {\text{Acc}}_t(g)|$. Notice that we can write
\begin{align}
    \widehat {\text{Acc}}_t(g)
    := \frac{1}{n(g)} \sum_{i=1}^{n(g)} \frac{1}{m} \sum_{j=1}^{m} I\{ g(x^{(t-i+1)}_j) = y^{(t-i+1)}_j\},
\end{align}

where $x_v^n$ is the $v^{th}$ covariate revealed at round $n$ in Framework \ref{fig:proto}. Here $n(g)$ is the final value of $r$ where the call to $\texttt{refineAccuracy}(g,t,\delta)$ stops. Since it depends on the hypothesis $g \in \cG$ handling this is different from the usual way of handling the excess risk in statistical learning theory \cite{bousquet2003lt} where the number of datapoints is independent of the hypothesis.

\textbf{observation 1}: To get around this issue, consider $mt$ datapoints: $S = \{(x^1,y^1)_{1:m}, \ldots,  (x^t,y^t)_{1:m}\}$. Consider two hypothesis $g,g' \in \cG$ such that $I\{g(x) = y \} = I \{ g'(x) = y\}$ for any $(x,y) \in S$. In such a case it follows that $n(g) = n(g')$ and $\widehat {\text{Acc}}_t(g) = \widehat {\text{Acc}}_t(g')$.

Now consider the loss-class $\cR = \{(x,y) \rightarrow I\{ g(x) = y\} : g \in \cG \}$ induced by $\cG$. Let $\cR'$ be the $mt$-shattering-set (Def.\ref{def:shatter}). Thus $|\cR'| =  S(\cR,mt)$. Notice that each hypothesis in $\cG$ is associated with a loss-composed hypothesis in $\cR'$. In view of observation 1, inorder to bound the loss-composed-hypothesis random variable $ |\widehat {\text{Acc}}_t(g) - {\text{Acc}}_t(g)|$ for any $g \in \cG$ it suffices to take a union bound of Proposition \ref{prop:mazeto} across all elements in $\cR'$. This implies that

\begin{align}
    |\widehat {\text{Acc}}_t(g) - {\text{Acc}}_t(g)|
    &= \tilde O(\sqrt{\log(S(\cR,mt))/mr}), \label{eq:shatter}
\end{align}
with probability at-least $1-\delta \cdot \log(S(\cR,mt))$ for any $g \in \cG$.

Each hypothesis in $\cG$ is a $d$-dimensional linear binary classifier. So VC dimension of $\cG$ is $d$. Since the VC dimension of $\cR$ is at-most twice that of $\cG$ we have that $\log(S(\cR,mt)) = 2d$. Further since $|\text{ACTIVE}(t)| \le 2 \log T$, combined with Sauer's lemma \cite{fml} we have that $\delta \cdot \log(S(\cR,mt)) \le 4\delta \log T \log mT$.

Now putting everything together results in Theorem \ref{thm:gen}.

\end{proof}

\begin{proposition} \label{prop:mazeto} (due to Theorem 1 in \cite{mazzetto2023adaptive})
For a fixed model $g \in \cG$ we have that  with probability at-least $1-\delta$
\begin{align}
    |{\text{Acc}}_t(g) - \widehat {\text{Acc}}_t(g)|
    &= \tilde O \left( \sqrt{\frac{1}{mr} }\right).
\end{align}
\end{proposition}
\begin{proof}
The proof is a direct consequence of the result in \cite{mazzetto2023adaptive}. For the sake of completeness, we show how it follows from Theorem 1 in \cite{mazzetto2023adaptive}.

Let $(\cX \times \mathcal Y, \cA)$ be a measurable space. By the assumption in Theorem \ref{thm:gen}, $(X_i, Y_i) \sim D_i$ are mutually independent random variables. 
Let $D_t^r := \frac{1}{r} \sum_{\tau = t-r+1}^t D_t(A)$ for all $A \in \cA$. Let $\hat D_t^r$ be the corresponding empirical distribution defined by

\begin{align}
    \hat D_t^r 
    &:= \frac{|(X_\tau,Y_\tau) \in A: t-r+1 \le \tau \le t|}{mr},
\end{align}

where we recall that there are $m$ hold-out data-points revealed after each round.

For a fixed model $g \in \cG$ consider the singleton function class $\cF := \{(x,y):\rightarrow I\{g(x) = y \} \}$. For any distributions $P,Q$ on $(\cX \times \mathcal Y, \cA)$ define:

\begin{align}
    P(g)
    &:= E_{(X,Y) \sim P}[I\{g(X) = Y \}],
\end{align}

and
\begin{align}
    \|P - Q \|_{\cF}
    &:= |P(g) - Q(g)|.
\end{align}

Due to Hoeffding's lemma, we have that for any fixed $r \le t$ and $\delta \in (0,1)$,

\begin{align}
    \|D_t^r - \hat D_t^r \|_{\cF} = O(\sqrt{\log(1/\delta)/(mr)}),
\end{align}
with probability at-least $1-\delta$. Hence Assumption 1 in \cite{mazzetto2023adaptive} holds.

Let $\hat r$ be the final value where the $\texttt{refineAccuracy}(g,t,\delta)$ procedure (Alg.\ref{fig:refine}) stops.
Now Theorem 1 in \cite{mazzetto2023adaptive} states that the output of $\texttt{refineAccuracy}(g,t,\delta)$ satisfies with probability $1-\delta$ that

\begin{align}
    \|D_t - \hat D_t^{\hat r} \|_{\cF}
    &= \tilde O \left( \min_{u \le t} \left[ \frac{1}{\sqrt{um}} + \max_{\tau < u} \|D_t - D_{t-\tau} \|_{\cF}\right] \right). \label{eq:eq1}
\end{align}

Using the fact that the data distribution has remained constant in $[t-r,t]$ as in the assumption of Theorem \ref{thm:gen} and upper bounding the minimum across $u$ by plugging in $u=r$, we get that

\begin{align}
    |\text{Acc}_t(g) - \widehat {\text{Acc}}_t(g)|
    &= \tilde O(1/\sqrt{mr}).
\end{align}

This completes the proof.

\end{proof}

In Section \ref{sec:theory}, our analysis focused on the case where the distribution shifts were assumed to be piece-wise stationary. Next, we relax that assumption and study the dynamic regret of (see Eq.\eqref{eq:dyn}) \texttt{AWE} under the cases of slowly evolving shifts.

\begin{theorem} \label{thm:dynamic}
Assume the notations used in Section \ref{sec:theory} and Theorem \ref{thm:gen}. Consider an arbitrary partitioning $\mathcal{P}$ of the time horizon into $M$ bins as $[i_s,i_t]$ for $i=1,\ldots,M$. Define $V_{i_s:i_t} := \max_{u \in [i_s,i_t]} TV(D_{i_t},D_u)$ which is the maximum total variation (from end time-point of the bin) of the data distribution within the $i^{th}$ bin.  Define $V := \sum_{i=1}^M V_{i_s:i_t+1}$. We have the following dynamic regret bound for \texttt{AWE}:
\begin{align}
    R_{\text{dynamic}}
    &= \sum_{k=1}^T \text{Acc}_k(h_k^*) - \text{Acc}_k(\hat f_k)\\
    &= \sum_{k=1}^T \text{Acc}_k(h_k^*) - \text{Acc}_k(f_k^*) +  \min_{M} \min_{\mathcal P} \sum_{i=1}^M \tilde O(\sqrt{dn_i/m} + n_i V_{i_s:i_t})),
\end{align}
with probability at-least $1-\delta$.
\end{theorem}
\begin{proof}

We continue our arguments from Eq.\eqref{eq:eq1}. Let $j \in [i_s,i_t]$. For any fixed hypothesis $g \in \cG$ (see Notations section before the proof of Theorem \ref{thm:gen} for the definition of $\cG$) we get with probability at-least $1-\delta$ that

\begin{align}
    |\text{Acc}_j(g) - \widehat {\text{Acc}}_j(g)|
    &= \tilde O(1/\sqrt{mj} + V_{i_s:i_t}).
\end{align}

By taking a union bound over the hypothesis class $\cG$ similar to Eq.\eqref{eq:shatter} and noting that the metric entropy of $\cG$ is $O(d)$ we conclude that

\begin{align}
    |\text{Acc}_j(g) - \widehat {\text{Acc}}_j(g)|
    &= \tilde O(\sqrt{d/mj} + V_{i_s:i_t}).
\end{align}

Define $n_i:= i_t - i_s + 1$ which the length of the $i^{th}$ bin. Next by taking a union bound across all time points within a bin (and after re-adjusting $\delta$) and continuing from the decomposition in Eq.\eqref{eq:decomp}, we have with probability at-least $1-\delta$ that 
\begin{align}
    \sum_{j=i_s}^{i_t} \text{Acc}_j(f_j^*) - \text{Acc}_j(\hat f_j)
    &= \sum_{j=i_s}^{i_t} \tilde O(\sqrt{d/mj} + V_{i_s:i_t})\\
    &= \tilde O(\sqrt{dn_i/m} + n_i V_{i_s:i_t}))
\end{align}

Now summing across all bins, using the decomposition in Eq.\eqref{eq:decomp1} and noting that the partition was selected arbitrarily results in the theorem.

\end{proof}

To better understand how the dynamic regret relates to the variation in the evolving data distributions, we provide the following corollary to Theorem \ref{thm:dynamic}.

\begin{corollary} \label{cor:dyn}
Assume the notations of Theorem \ref{thm:dynamic}. We have with probability at-least $1-\delta$ that \texttt{AWE} satisfies
\begin{align}
    R_{\text{dynamic}}
    &=  \tilde O(T^{2/3} (V/m)^{1/3} + \sqrt{T/m}) +  \sum_{k=1}^T \text{Acc}_k(h_k^*) - \text{Acc}_k(f_k^*)
\end{align}
\end{corollary}
\begin{remark}
The expression in the RHS of the bound in Corollary \ref{cor:dyn} is composed of two terms. The first term characterizes the dynamic regret. The second term characterises the approximation error of the instance pool $\{E_t\} \cup \cA_t$ (see \texttt{AWE} Alg.\ref{fig:awe} for definitions of $E_t$ and $\cA_t$) in approximating the best hypothesis $h_t^* \in \cH$. Since $E_t$ is an ensemble model which can be more expressive than individual instances of the OL algorithm trained on disjoint pieces of the history, the second term can be potentially negative as well.
\end{remark}

\begin{proof}
Note that Theorem \ref{thm:dynamic} holds for any partitioning schemes. Hence to further upperbound the dynamic regret, we can compute the bound in Theorem \ref{thm:dynamic} for any specific partitioning scheme.

Consider the following partitioning of the time horizon into $M$ bins such that $V_{i_s:i_t} \le \epsilon$ while $V_{i_s:i_t+1} > \epsilon$. Suppose that $M > 1$. We have the following bound on number of bins

\begin{align}
    V &= \sum_{i=1}^M V_{i_s:i_t+1}\\
    &\ge \epsilon M.
\end{align}

The number of bins also must be at-least 1. Hence the number of bins obeys $M \le 1 + V/\epsilon$. Instantiating Theorem \ref{thm:dynamic} with the above partitioning results in

\begin{align}
    R_{\text{dynamic}}
    &= \sum_{k=1}^T \text{Acc}_k(h_k^*) - \text{Acc}_k(\hat g_k)\\
    &\le \sum_{k=1}^T \text{Acc}_k(h_k^*) - \text{Acc}_k(g_k^*) + \sum_{i=1}^M \tilde O(\sqrt{dn_i/m} + n_i V_{i_s:i_t}))\\
    &\le_{(a)} \sum_{k=1}^T \text{Acc}_k(h_k^*) - \text{Acc}_k(g_k^*) + \tilde O(\sqrt{dMT/m}) + T\epsilon\\
    &\le \sum_{k=1}^T \text{Acc}_k(h_k^*) - \text{Acc}_k(g_k^*) + \tilde O(\sqrt{(dT/m) (1+V/\epsilon) }) + T\epsilon\\
    &\le \sum_{k=1}^T \text{Acc}_k(h_k^*) - \text{Acc}_k(g_k^*) + \tilde O(\sqrt{dT/m}) + \tilde O (\sqrt{VT/(m\epsilon)}), 
\end{align}
where line (a) is due to Cauchy Schwartz. Optimizing over $\epsilon$ and setting $\epsilon = (V/(mT))^{1/3} $ yields the theorem.

\end{proof}

Next, we compare our bound with existing dynamic regret bounds in the literature.
\begin{remark}
We note that our bound is different and not directly comparable to a dynamic regret bound of the form $O(\sqrt{T(P+1)})$ where $P = \sum_{t=2}^T TV(D_t, D_{t-1})$ \cite{zhang2018adaptive}. The reasons are as follows: 1) In the batched online setting of Fig.\ref{fig:proto}, the value of $m$ (the number of points in the hold-out data set at each round) can be substantially larger than 1 as is the case with our experiments. 2) Though the variational $V$ can be bounded above by $P$, such a bound can be very loose. So the value $V$ can be much smaller than $P$. Hence in the batched online framework of Fig.\ref{fig:proto}, the bound given by Corollary \ref{cor:dyn} can be tighter than a regret bound of the form $O(\sqrt{T(P+1)})$.
\end{remark}

\section{Comparison to \cite{daniely2015strongly}} \label{app:daniely}
In this section, we provide a close comparison between our work and that of \cite{daniely2015strongly} both of which are black-box adaptation techniques.

\subsection{Failure mode of Geometric Covering (GC) intervals} \label{app:daniely1}
GC intervals developed in \cite{daniely2015strongly} fails to satisfy a data coverage guarantee as stipulated by Theorem \ref{thm:mri}. We describe a minimalistic scenario where GC is insufficient to give a data coverage guaranteed by MRI as in Theorem 3.

Let $T=10$. The GC intervals that spans this time horizon can be split into various resolution as follows.

$Res_0 = {[1,1],[2,2],\ldots,[10,10] }$ $Res_1 = {[2,3],[4,5],[6,7],[8,9], [10,11] }$ $Res_2 = {[4,7],[8,11]}$ $Res_3 = {[8,15]}$

Suppose The distribution shift is such that data at time $1$ is generated from distribution $D_1$ and times $[2,10]$ are generated from another distribution $D_2$. Let the number of labelled examples revealed after each online round be $N$. Suppose we are the beginning of the online round $t=9$. So we have seen $7N$ labelled data points from distribution $D_2$.

We have $\text{ACTIVE}(9) = {[9,9],[8,9],[8,11],[8,15] }$. Since all the active intervals start from timepoint $8$, the experts defined by these active intervals have only seen $N$ labelled data points from distribution $D_2$ when we are at round $9$. Since $N < 7N/4 < 7N/2$, it fails to provide a data coverage guarantee as stipulated by MRI in Theorem 3. On the other hand, in this example there exists a bracket in the MRI from the $B$ set that can cover $4N$ data points from distribution $D_2$. The under-coverage effect of GC can be more exaggerated when we consider longer time horizons.

\subsection{Differences in problem setting and regret guarantees} \label{app:daniely2}
We note that the setting considered in our paper differs from that of the usual black-box model selection in the adaptive online learning literature \cite{daniely2015strongly}. We consider the setting of batched online learning (Framework \ref{fig:proto}) where the learner makes predictions for the labels of a collection of covariates (say $N$ covariates) that are revealed at each online round with true labels revealed only after all the predictive labels are submitted. This is suited for real-world usecases where it is practical to receive a collective feedback. The main message is that by using our methods, one can attain faster average interval regret guarantees than that promised by SAOL \cite{daniely2015strongly}. We explain this in detail below.

Conventional online algorithms operate in the regime of $N=1$. However, if we define the loss suffered at a round as the average loss incurred by the algorithm for all covariates, then we can reduce the setting of Framework \ref{fig:proto} to that studied in \cite{daniely2015strongly} and use the black-box model selection algorithms studied there. But this approach can lead to serious drawbacks when applied to our setting of batched online learning: Suppose in an interval $I \subseteq [T]$, the data distribution is constant. Then the Strongly Adaptive guarantees in \cite{daniely2015strongly} will lead to an average regret wrt the best model within interval $I$ as $O(1/\sqrt{|I|} + 1/\sqrt{N}) = O(1/\sqrt{|I|})$ (where for later inequality, we suppose $N > |I|$). Here regret is measured wrt population level accuracy as in Theorem \ref{thm:gen} and the term $1/\sqrt{N}$ is the artifact of concentration inequality to relate the empirical average loss at a round to its population analogue. However, by summing up the regret bound in Theorem \ref{thm:gen}, our model selection scheme lead to an average regret of $O(1/\sqrt{p \cdot N|I|})$ where $p$ is the validation-train split ratio in \texttt{AWE}. Since $p$ is selected such that $pN > 1$, such a rate leads to faster convergence than $O(1/\sqrt{|I|})$. As explained in notes on technical novelties in Section \ref{sec:intro}, this improved effect is attained by finding a weighted combination of active models with more weights allotted to models with best validation scores for the most recent distribution. Hence the distribution of the weights in our algorithm is more tilted (in comparison to that of SAOL in \cite{daniely2015strongly}) towards models with high recent validation score. The CVTT component helps to obtain high accuracy validation scores for each model.

\section{Additional Experimental Results on \texttt{AWE}} \label{app:exp}

\textbf{Experiments with \texttt{AWE}. }
We report the accuracy differences across various time stamps for different base OL algorithms. The experimental setting is same as that of the one described in Sec. \ref{sec:exp}. The results are displayed in Figures \ref{fig:images_ft},\ref{fig:images_ewc},\ref{fig:images_coral}, \ref{fig:images_irm} and Table \ref{tab:awe_win}. The trends noticed in Sec. \ref{sec:exp} remain to hold uniformly. 

\textbf{Experiments with SAOL. }
The experimental results on per-step accuracy gain with SAOL black-box scheme from \cite{daniely2015strongly} is shown in Figures \ref{fig:images_ft_gc},\ref{fig:images_ewc_gc},\ref{fig:images_coral_gc}, \ref{fig:images_irm_gc} and Table \ref{tab:other_win}.

\begin{figure*}[h]
    \centering
    \subfloat[FMOW]{\includegraphics[width=0.25\linewidth]{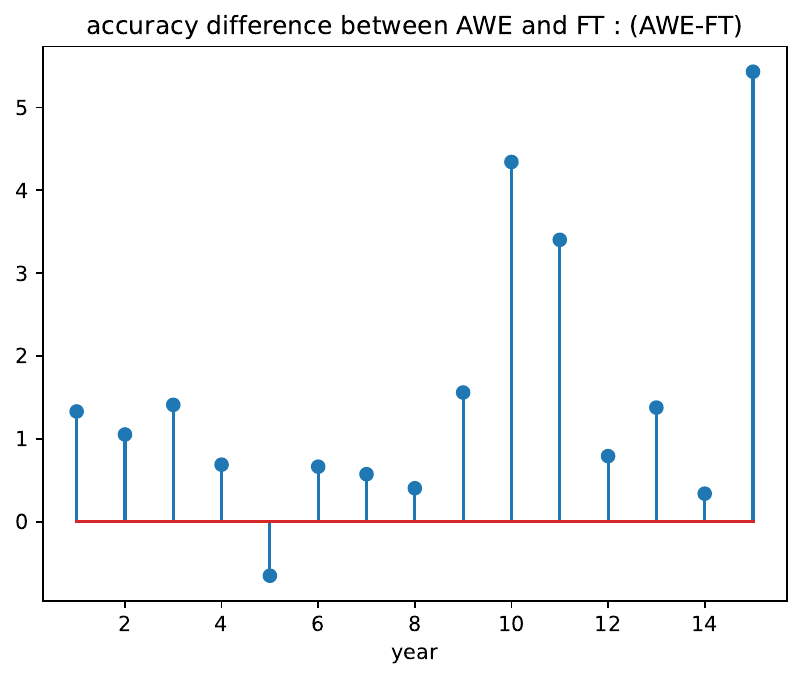}}\hfill
    \subfloat[Huffpost]{\includegraphics[width=0.25\linewidth]{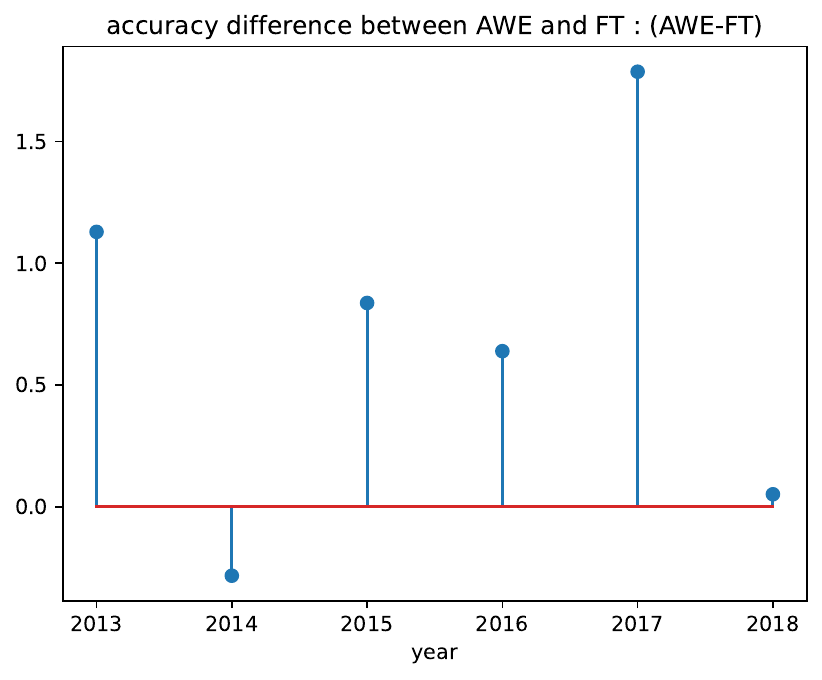}}\hfill
    \subfloat[Arxiv]{\includegraphics[width=0.25\linewidth]{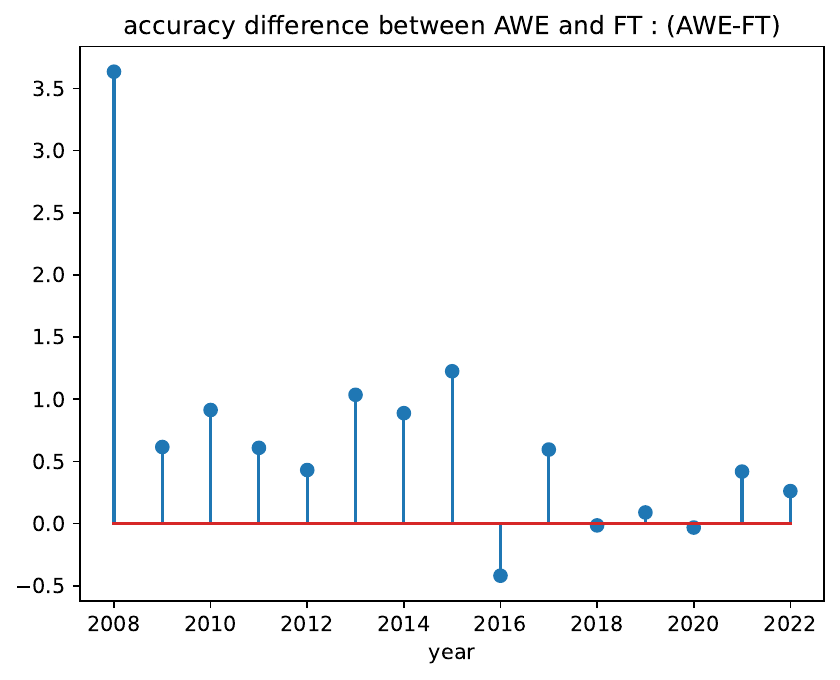}}

    \caption{$\%$ accuracy differences across various timestamps when \texttt{AWE} is run with FT as the online learning algorithm.}
    \label{fig:images_ft}
\end{figure*}

\begin{figure*}[htbp]
    \centering
    \subfloat[FMOW]{\includegraphics[width=0.25\linewidth]{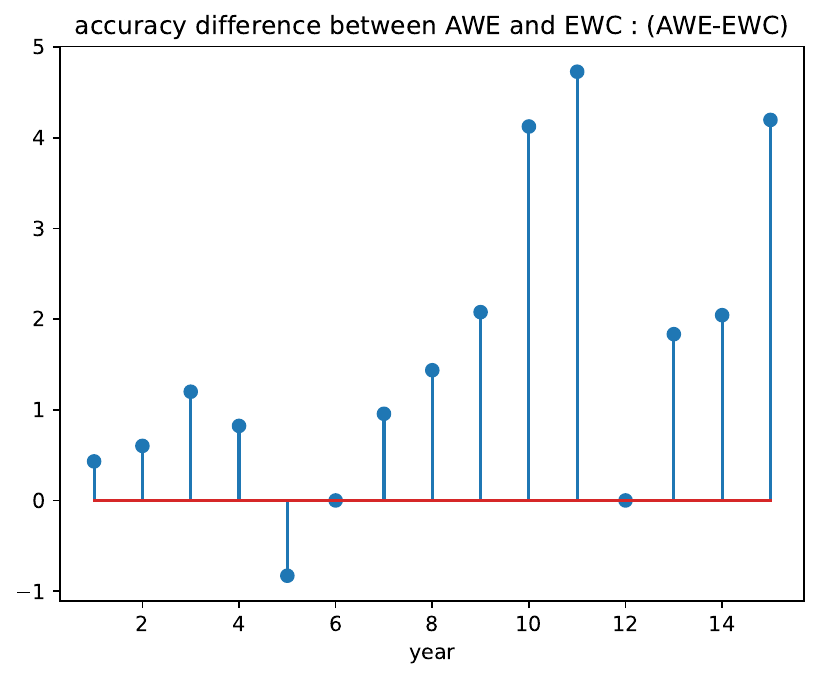}}\hfill
    \subfloat[Huffpost]{\includegraphics[width=0.25\linewidth]{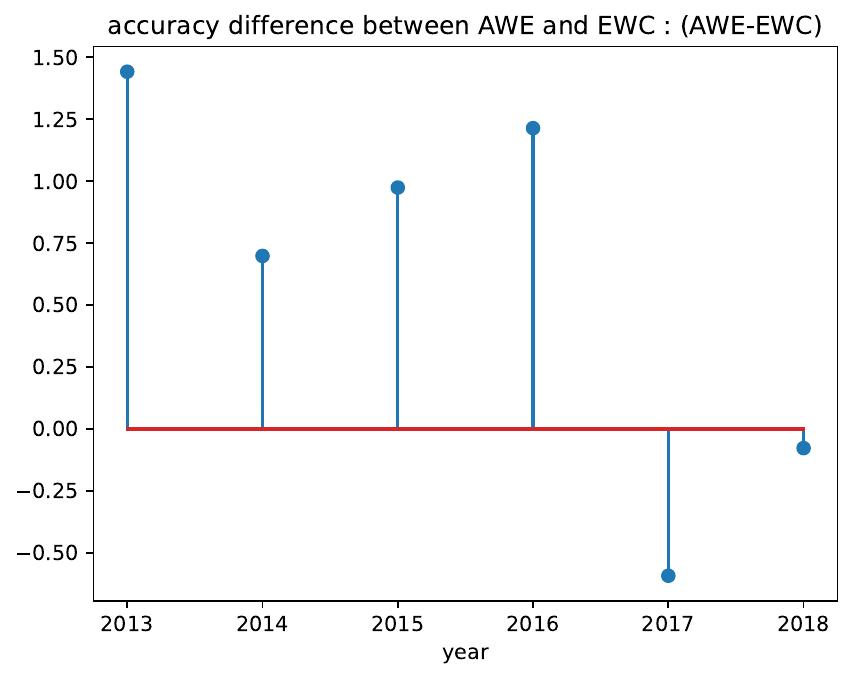}}\hfill
    \subfloat[Arxiv]{\includegraphics[width=0.25\linewidth]{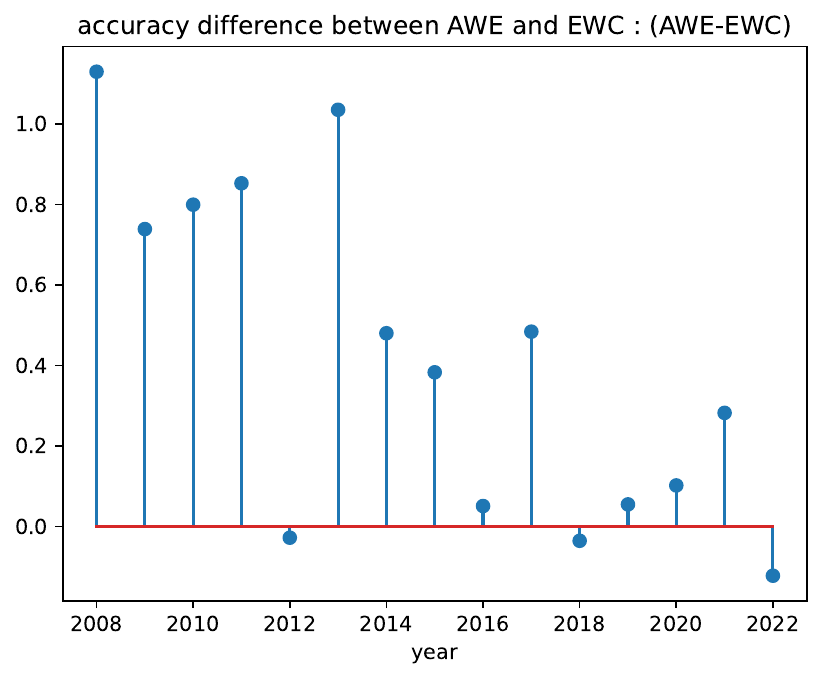}}

    \caption{$\%$ accuracy differences across various timestamps when \texttt{AWE} is run with EWC as the online learning algorithm.}
    \label{fig:images_ewc}
\end{figure*}

\begin{figure*}[htbp]
    \centering
    \subfloat[FMOW]{\includegraphics[width=0.25\linewidth]{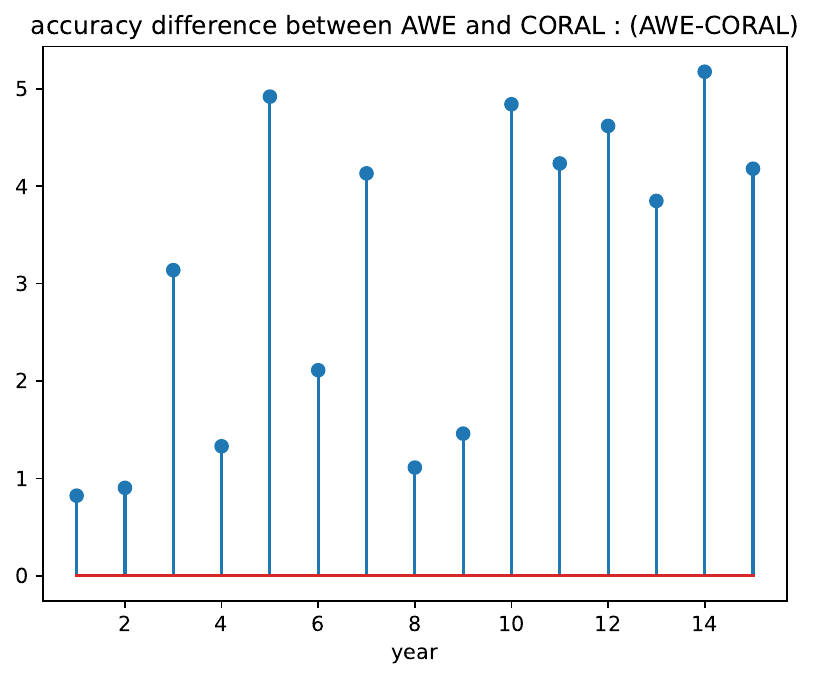}}\hfill
    \subfloat[Huffpost]{\includegraphics[width=0.25\linewidth]{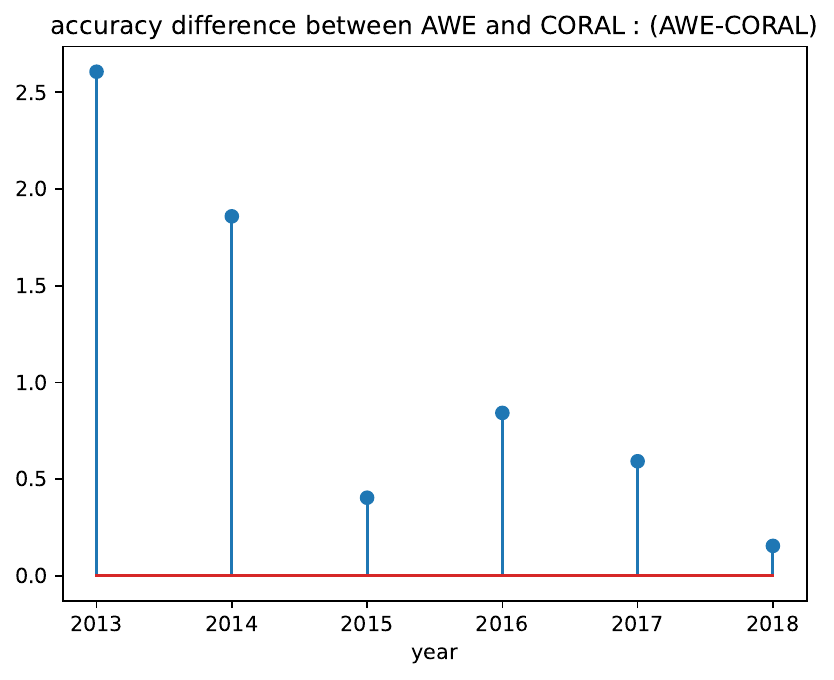}}\hfill
    \subfloat[Arxiv]{\includegraphics[width=0.25\linewidth]{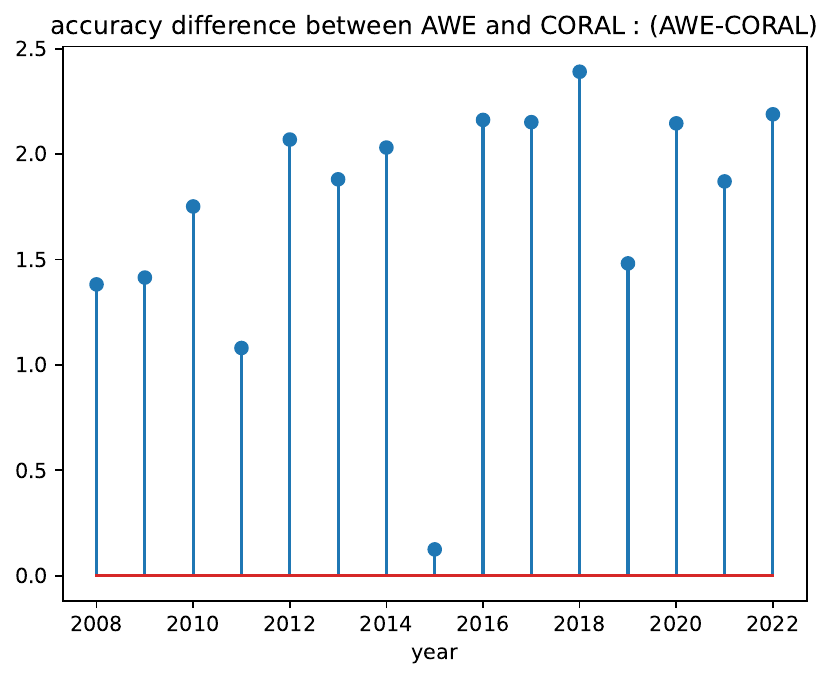}}

    \caption{$\%$ accuracy differences across various timestamps when \texttt{AWE} is run with CORAL as the online learning algorithm.}
    \label{fig:images_coral}
\end{figure*}

\begin{figure*}[htbp]
    \centering
    \subfloat[FMOW]{\includegraphics[width=0.25\linewidth]{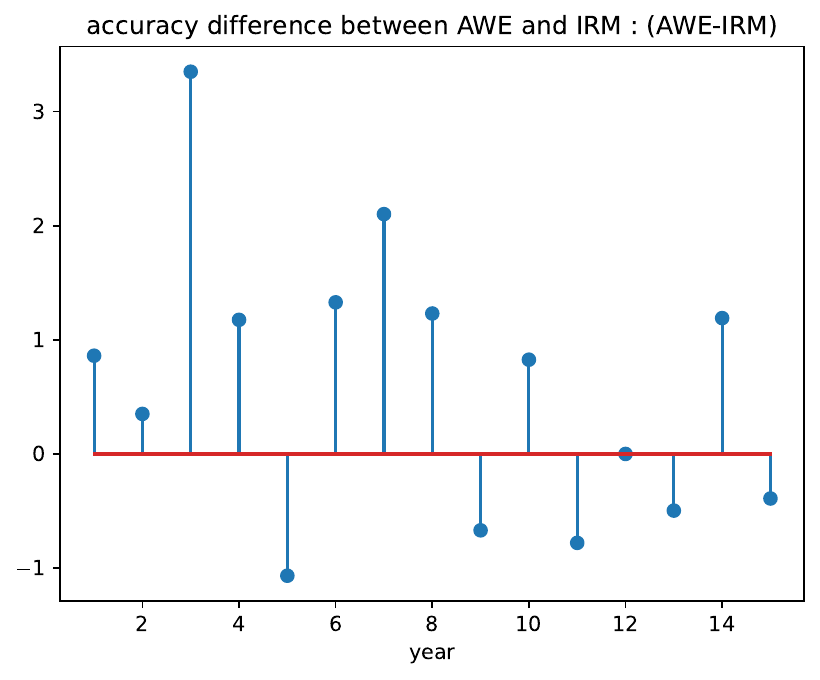}}\hfill
    \subfloat[Huffpost]{\includegraphics[width=0.25\linewidth]{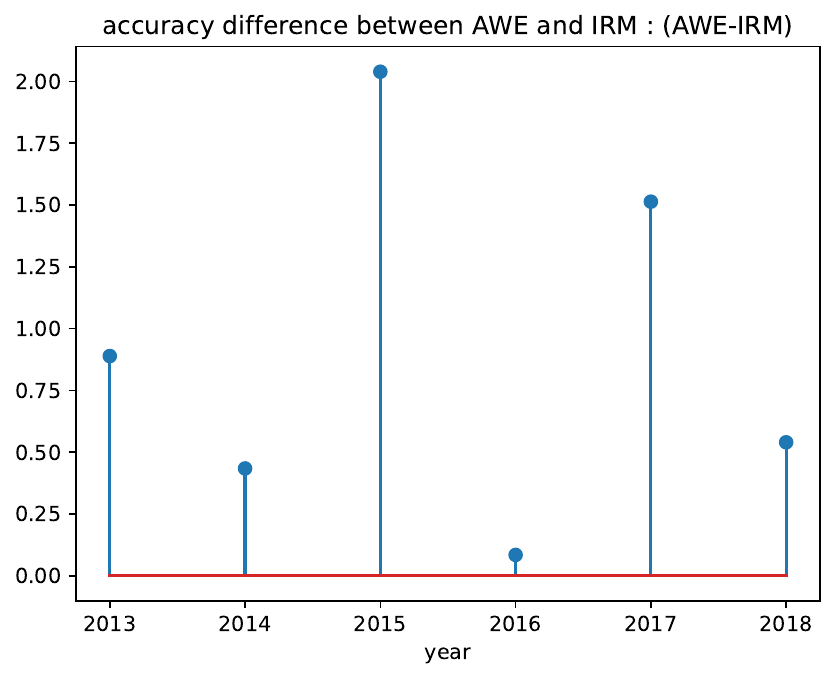}}\hfill
    \subfloat[Arxiv]{\includegraphics[width=0.25\linewidth]{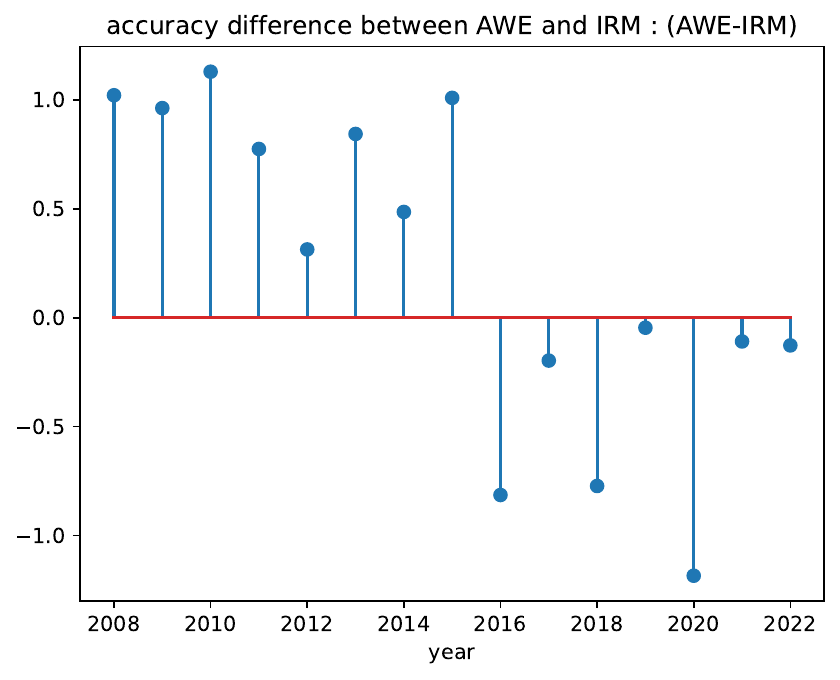}}

    \caption{$\%$ accuracy differences across various timestamps when \texttt{AWE} is run with IRM as the online learning algorithm.}
    \label{fig:images_irm}
\end{figure*}


\begin{figure*}[h]
    \centering
    \subfloat[FMOW]{\includegraphics[width=0.25\linewidth]{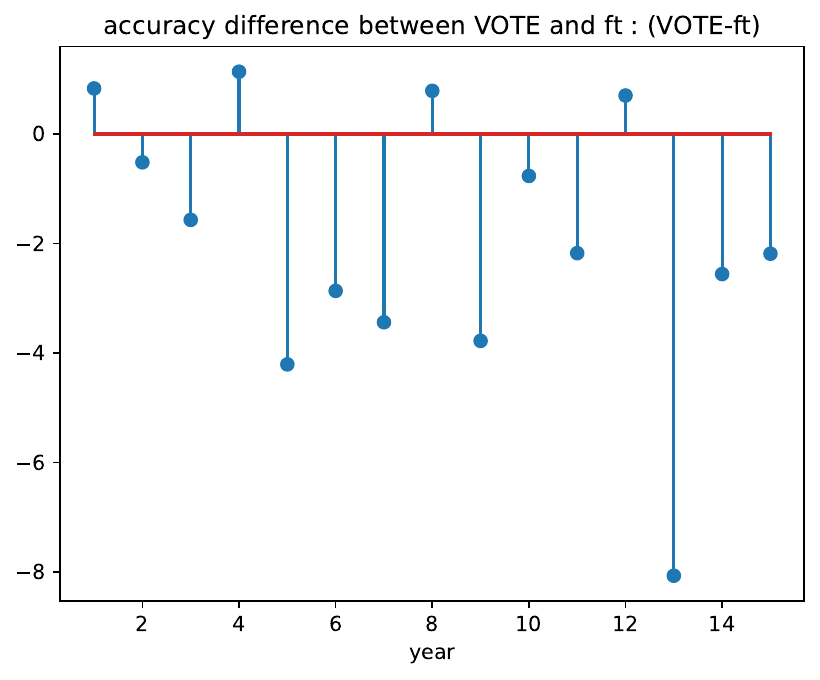}}\hfill
    \subfloat[Huffpost]{\includegraphics[width=0.25\linewidth]{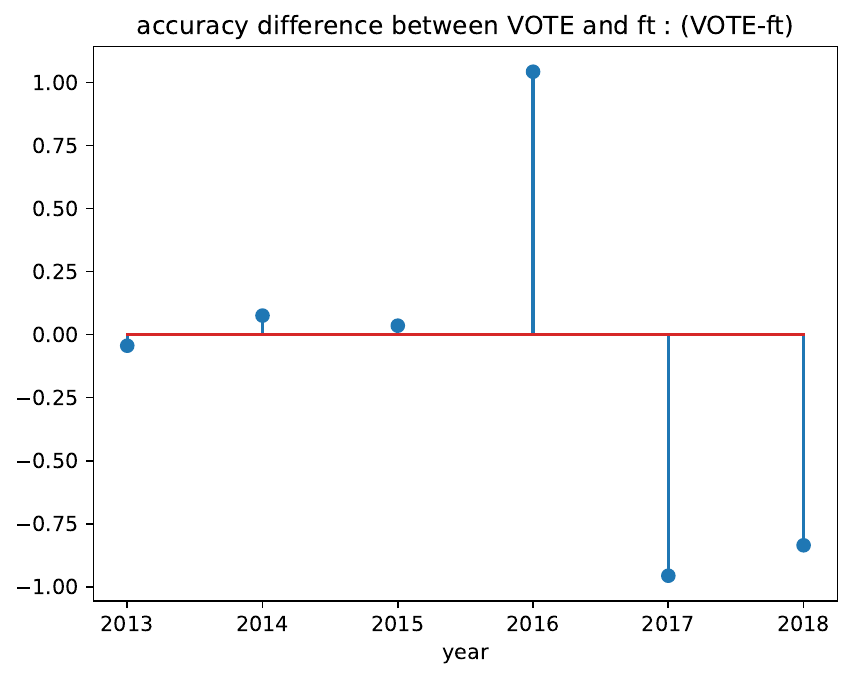}}\hfill
    \subfloat[Arxiv]{\includegraphics[width=0.25\linewidth]{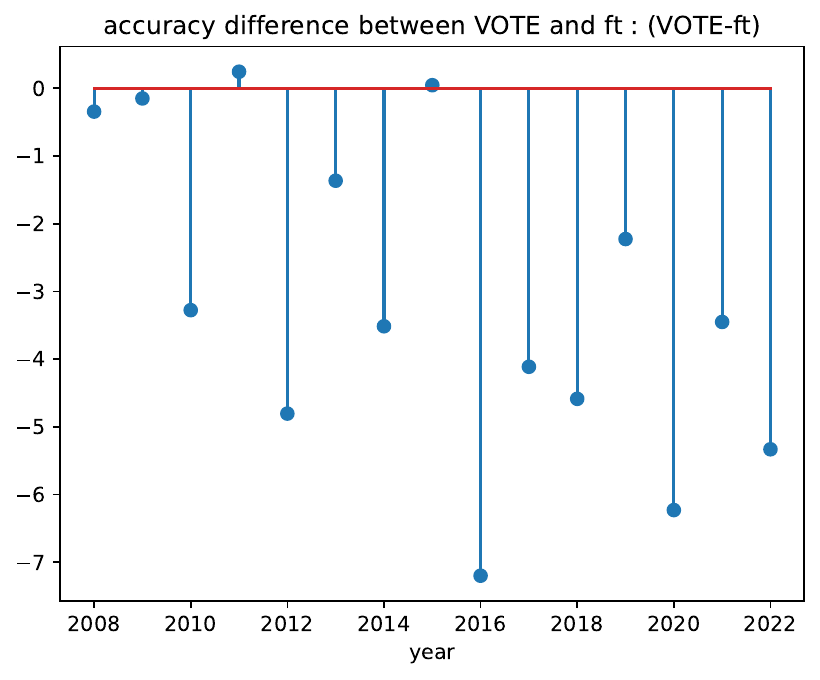}}

    \caption{$\%$ accuracy differences across various timestamps when \texttt{AWE} with majority voting is run with FT as the online learning algorithm.}
    \label{fig:images_ft_vote}
\end{figure*}

\begin{figure*}[h]
    \centering
    \subfloat[FMOW]{\includegraphics[width=0.25\linewidth]{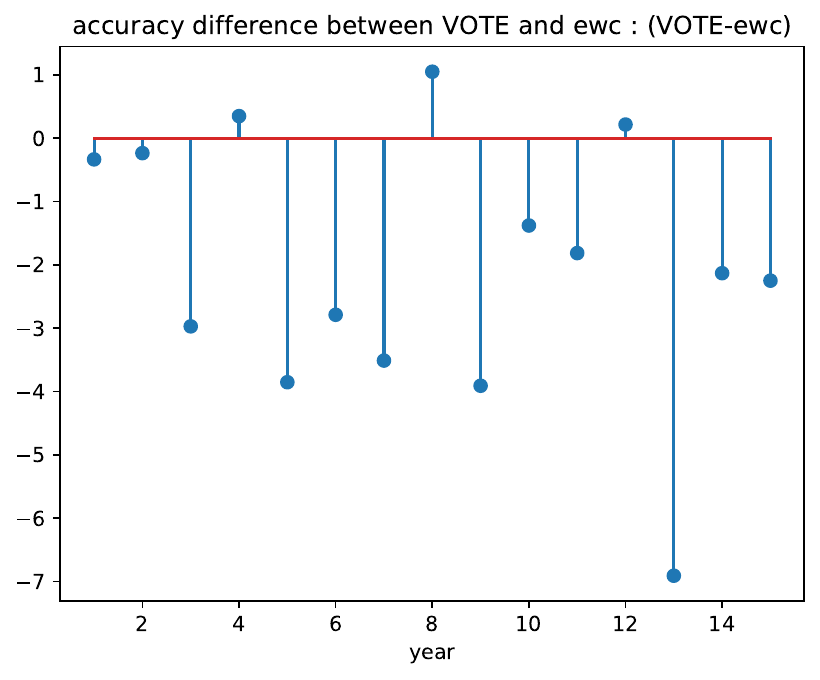}}\hfill
    \subfloat[Huffpost]{\includegraphics[width=0.25\linewidth]{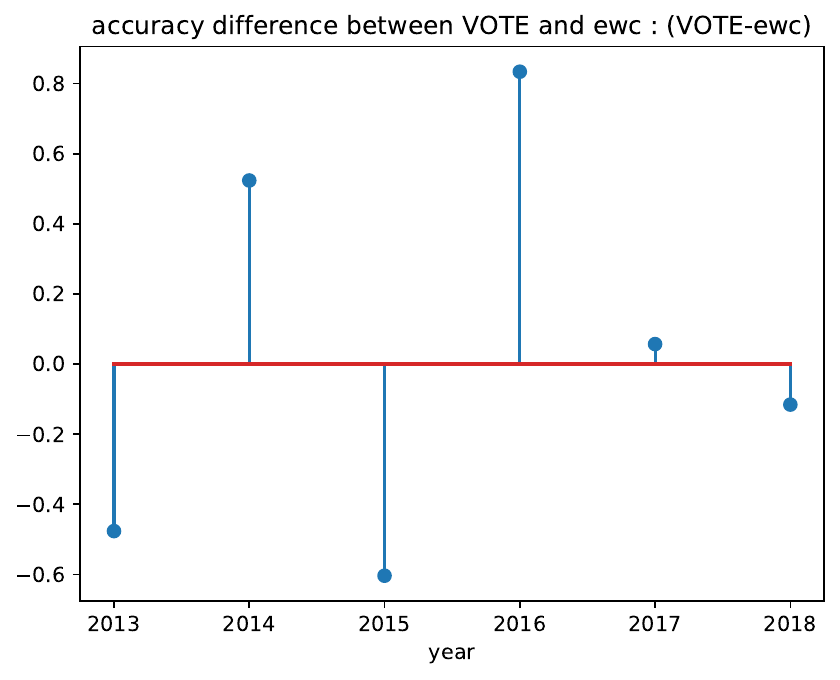}}\hfill
    \subfloat[Arxiv]{\includegraphics[width=0.25\linewidth]{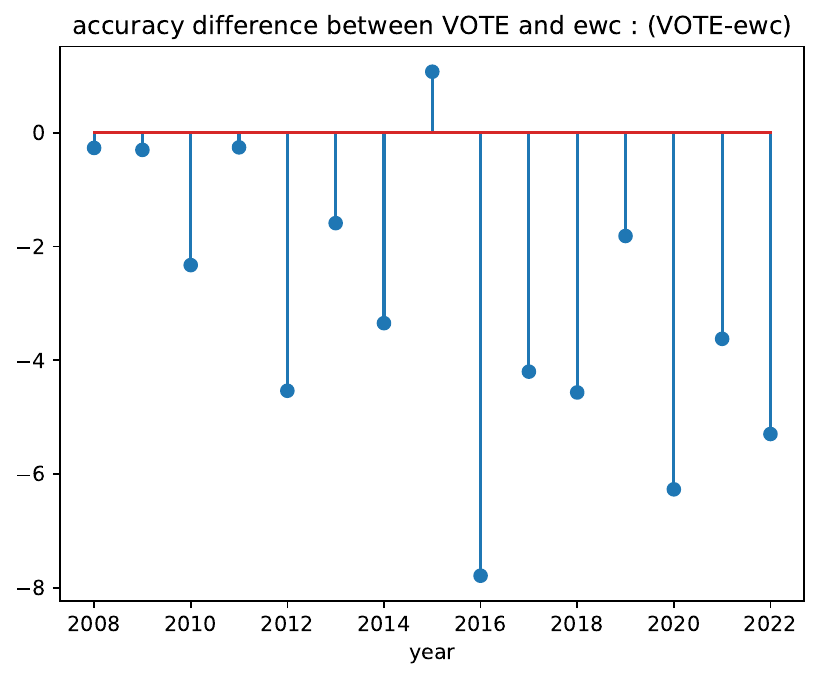}}

    \caption{$\%$ accuracy differences across various timestamps when \texttt{AWE} with majority voting is run with EWC as the online learning algorithm.}
    \label{fig:images_ewc_vote}
\end{figure*}

\begin{figure*}[h]
    \centering
    \subfloat[FMOW]{\includegraphics[width=0.25\linewidth]{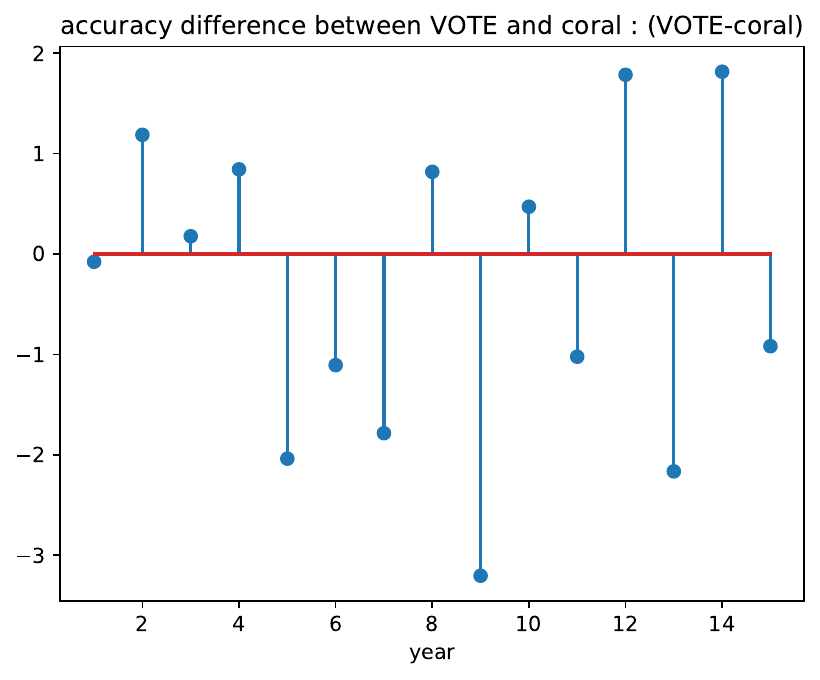}}\hfill
    \subfloat[Huffpost]{\includegraphics[width=0.25\linewidth]{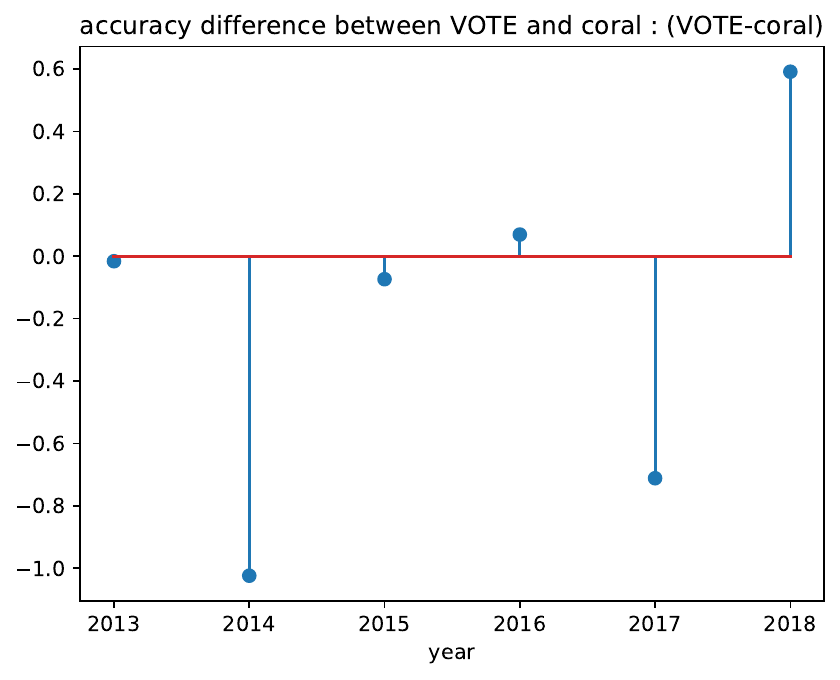}}\hfill
    \subfloat[Arxiv]{\includegraphics[width=0.25\linewidth]{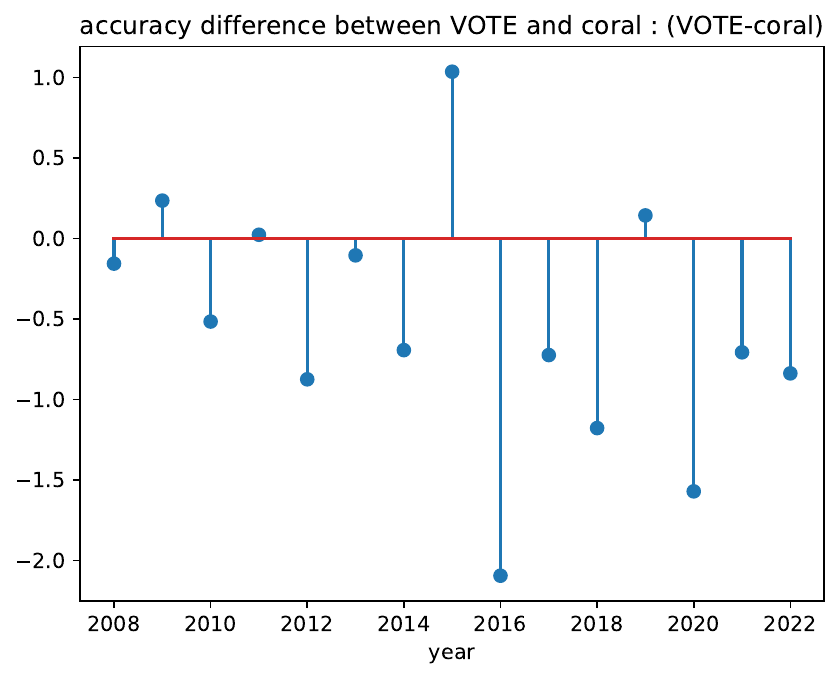}}

    \caption{$\%$ accuracy differences across various timestamps when \texttt{AWE} with majority voting is run with CORAL as the online learning algorithm.}
    \label{fig:images_coral_vote}
\end{figure*}

\begin{figure*}[h]
    \centering
    \subfloat[FMOW]{\includegraphics[width=0.25\linewidth]{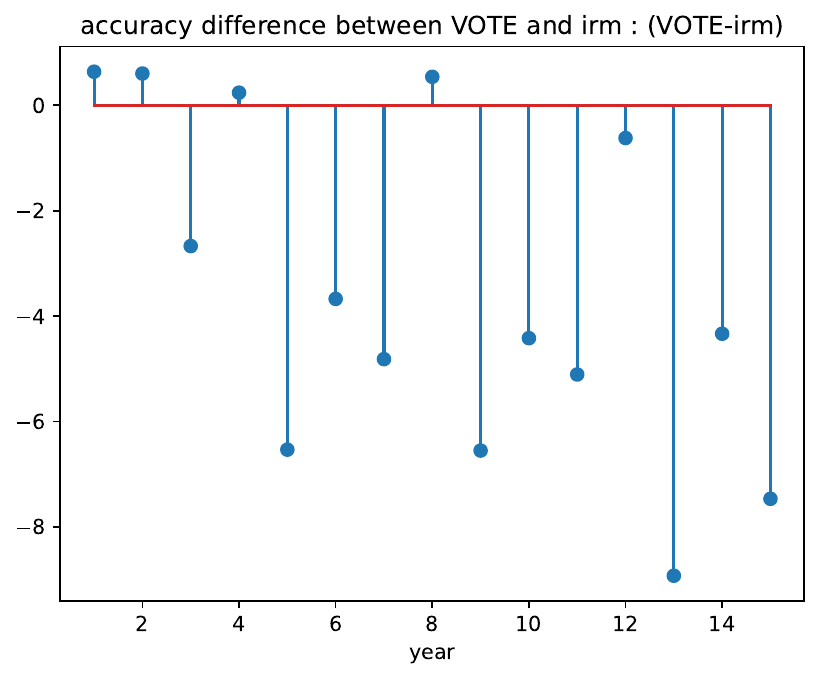}}\hfill
    \subfloat[Huffpost]{\includegraphics[width=0.25\linewidth]{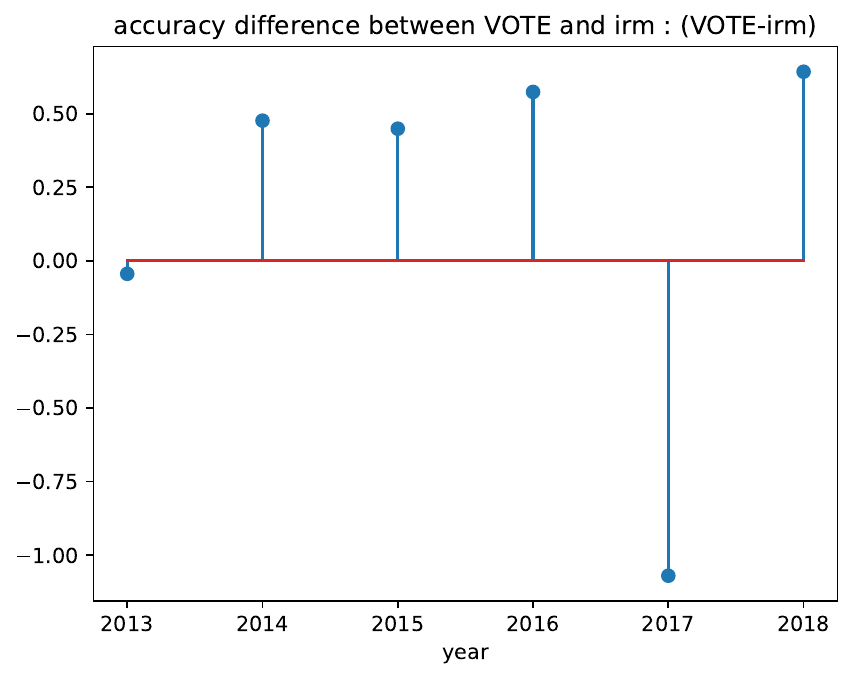}}\hfill
    \subfloat[Arxiv]{\includegraphics[width=0.25\linewidth]{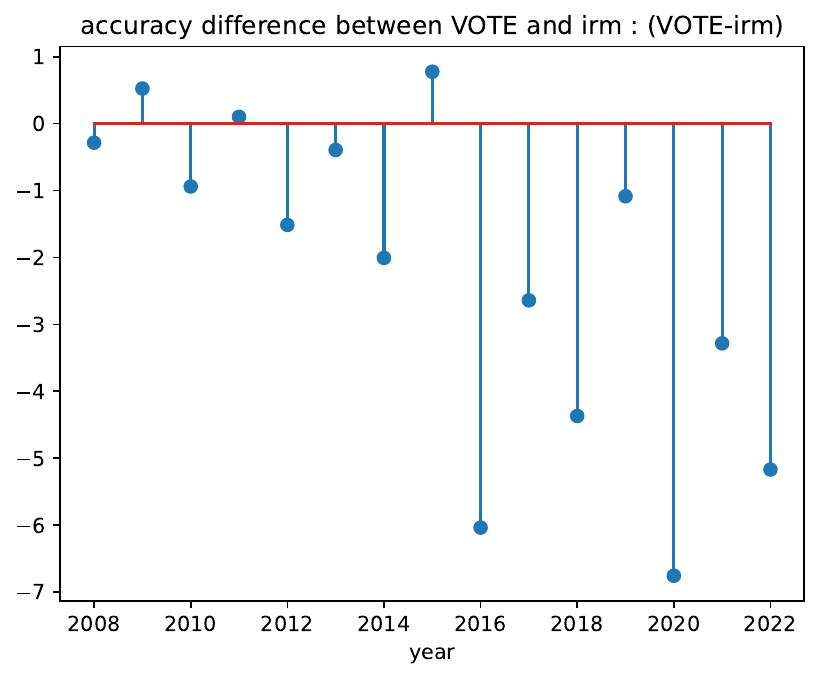}}

    \caption{$\%$ accuracy differences across various timestamps when \texttt{AWE} with majority voting is run with IRM as the online learning algorithm.}
    \label{fig:images_irm_vote}
\end{figure*}

\begin{figure*}[h]
    \centering
    \subfloat[FMOW]{\includegraphics[width=0.25\linewidth]{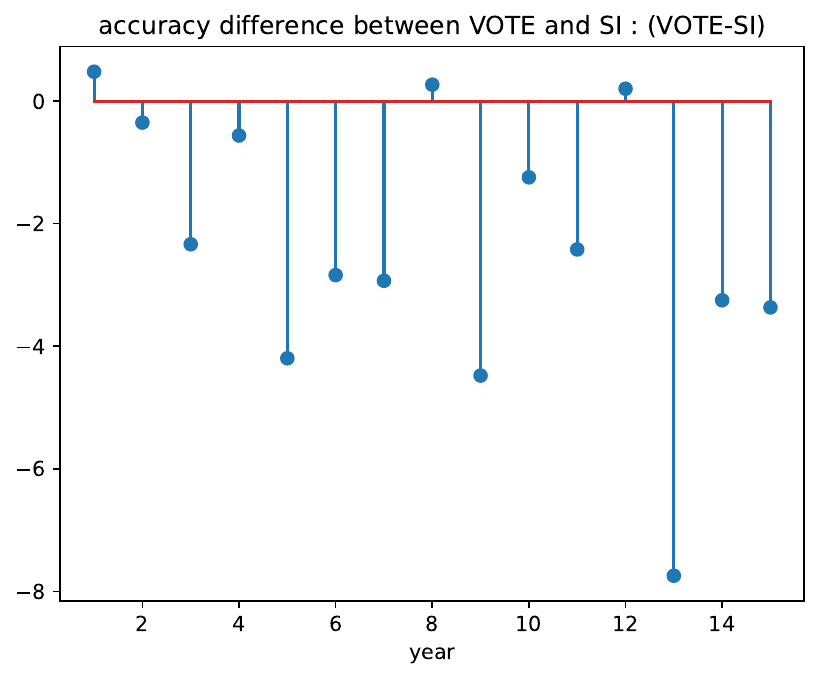}}\hfill
    \subfloat[Huffpost]{\includegraphics[width=0.25\linewidth]{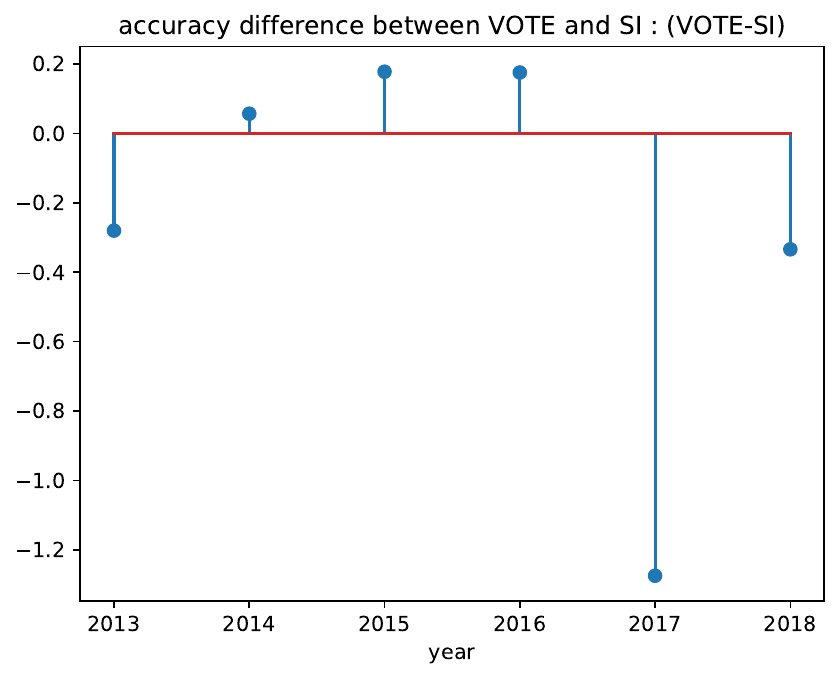}}\hfill
    \subfloat[Arxiv]{\includegraphics[width=0.25\linewidth]{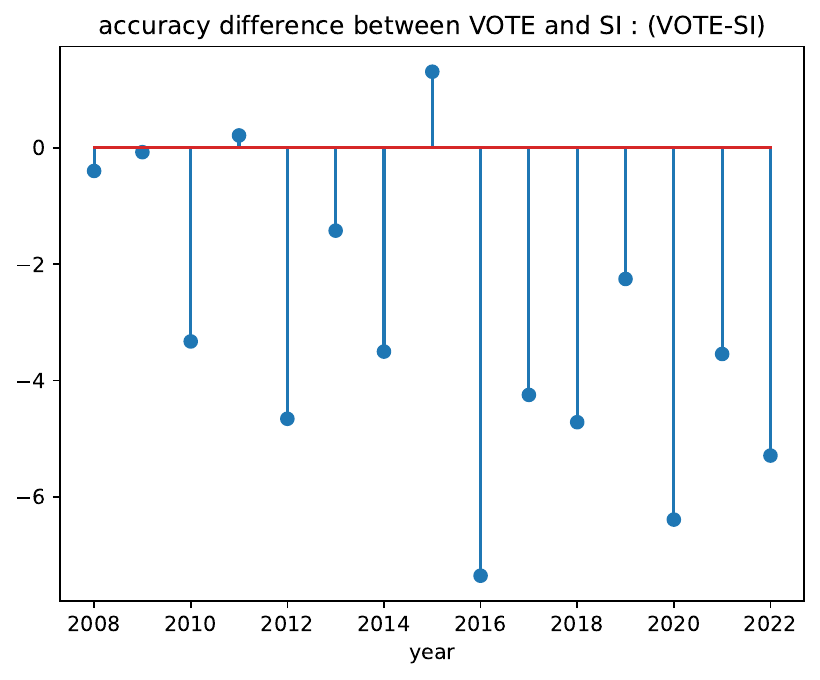}}

    \caption{$\%$ accuracy differences across various timestamps when \texttt{AWE} with majority voting is run with SI as the online learning algorithm.}
    \label{fig:images_si_vote}
\end{figure*}


\begin{figure*}[h]
    \centering
    \subfloat[FMOW]{\includegraphics[width=0.25\linewidth]{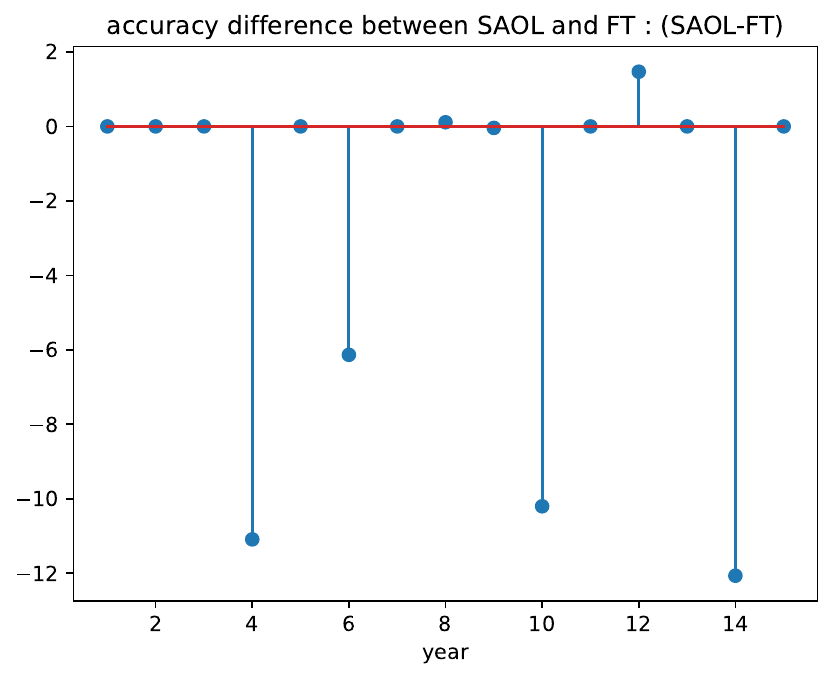}}\hfill
    \subfloat[Huffpost]{\includegraphics[width=0.25\linewidth]{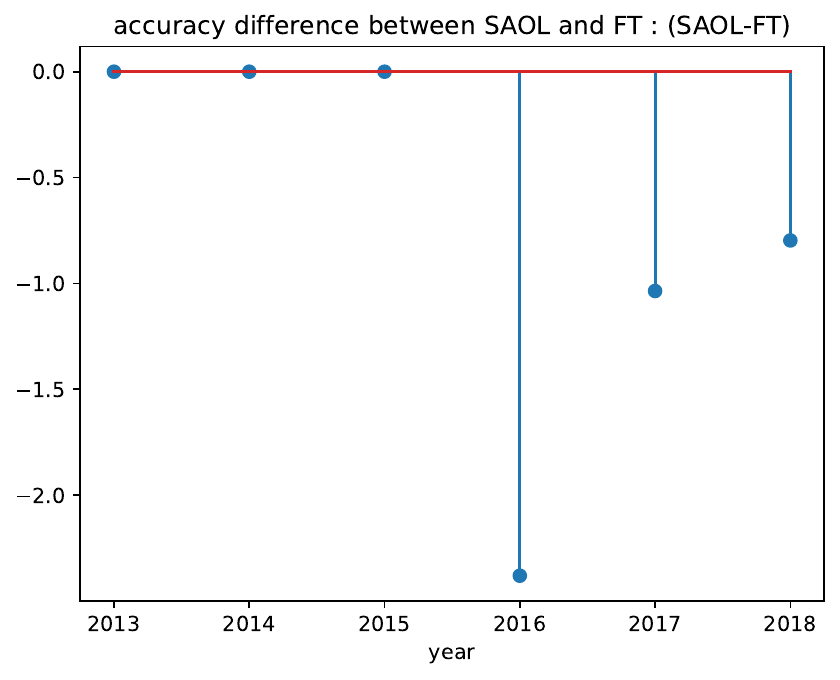}}\hfill
    \subfloat[Arxiv]{\includegraphics[width=0.25\linewidth]{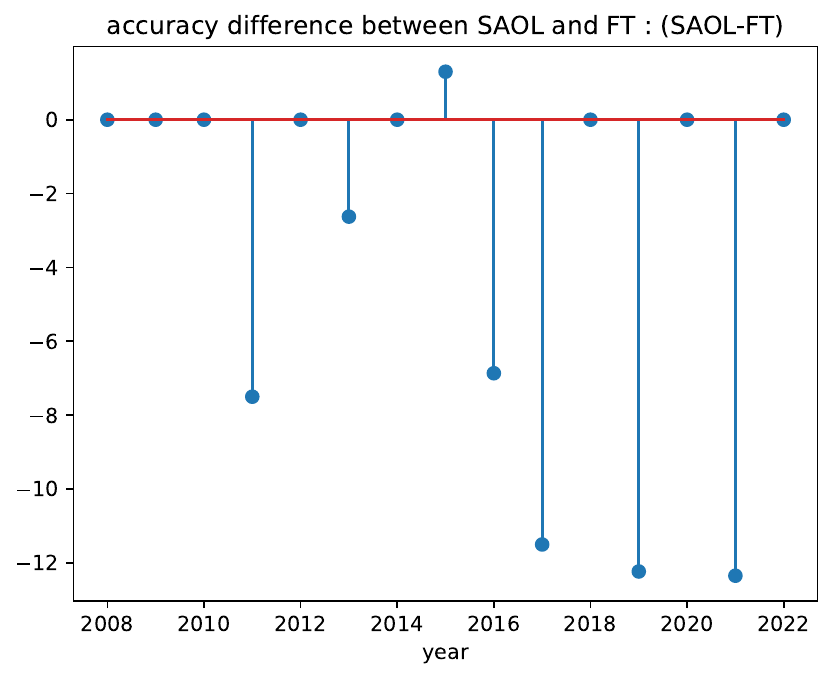}}

    \caption{$\%$ accuracy differences across various timestamps when SAOL is run with FT as the online learning algorithm.}
    \label{fig:images_ft_gc}
\end{figure*}

\begin{figure*}[h]
    \centering
    \subfloat[FMOW]{\includegraphics[width=0.25\linewidth]{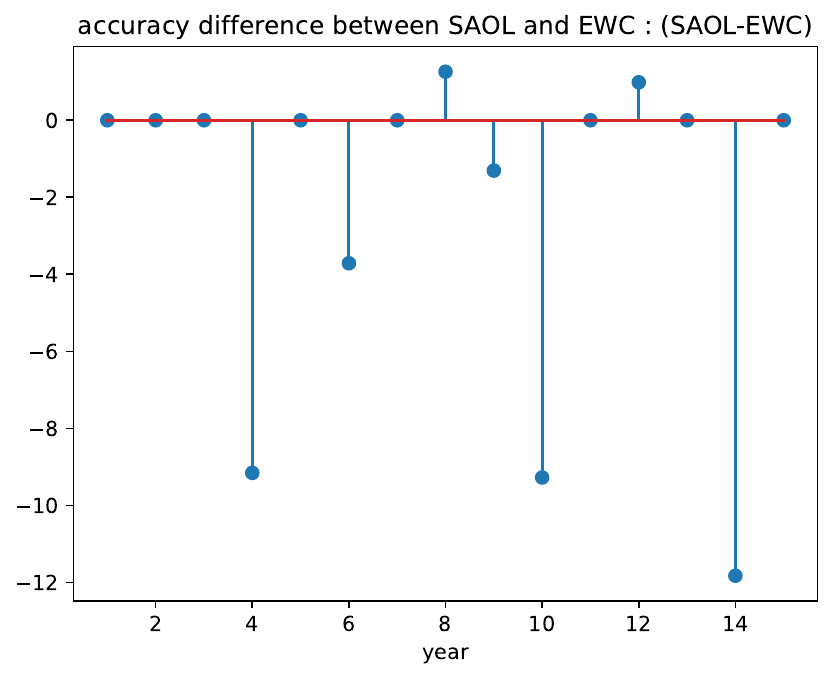}}\hfill
    \subfloat[Huffpost]{\includegraphics[width=0.25\linewidth]{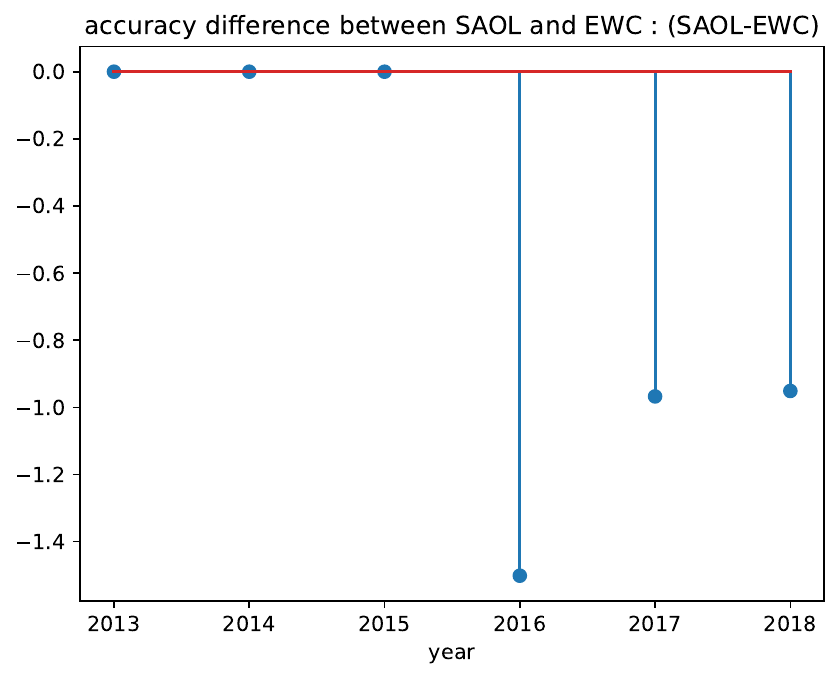}}\hfill
    \subfloat[Arxiv]{\includegraphics[width=0.25\linewidth]{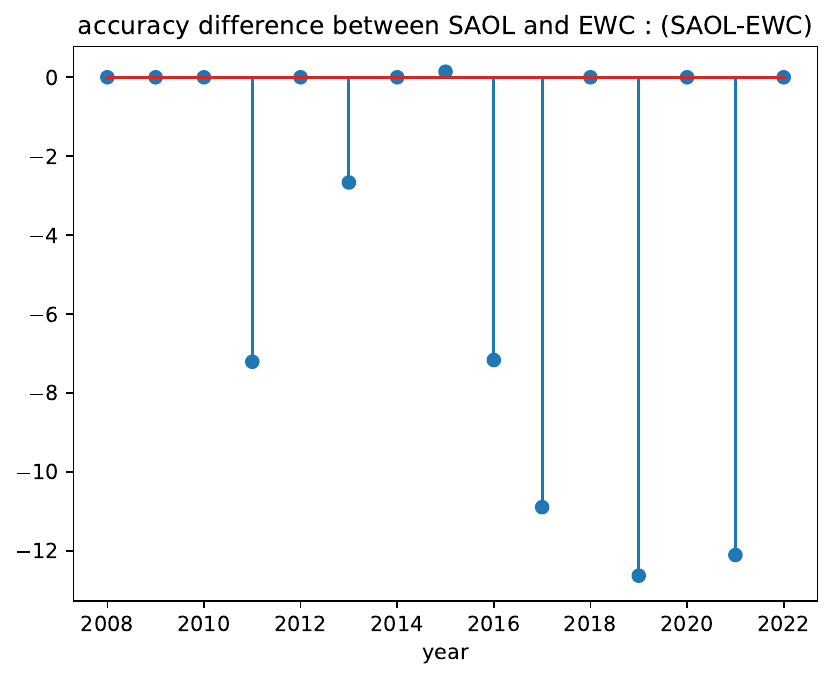}}

    \caption{$\%$ accuracy differences across various timestamps when SAOL is run with EWC as the online learning algorithm.}
    \label{fig:images_ewc_gc}
\end{figure*}

\begin{figure*}[h]
    \centering
    \subfloat[FMOW]{\includegraphics[width=0.25\linewidth]{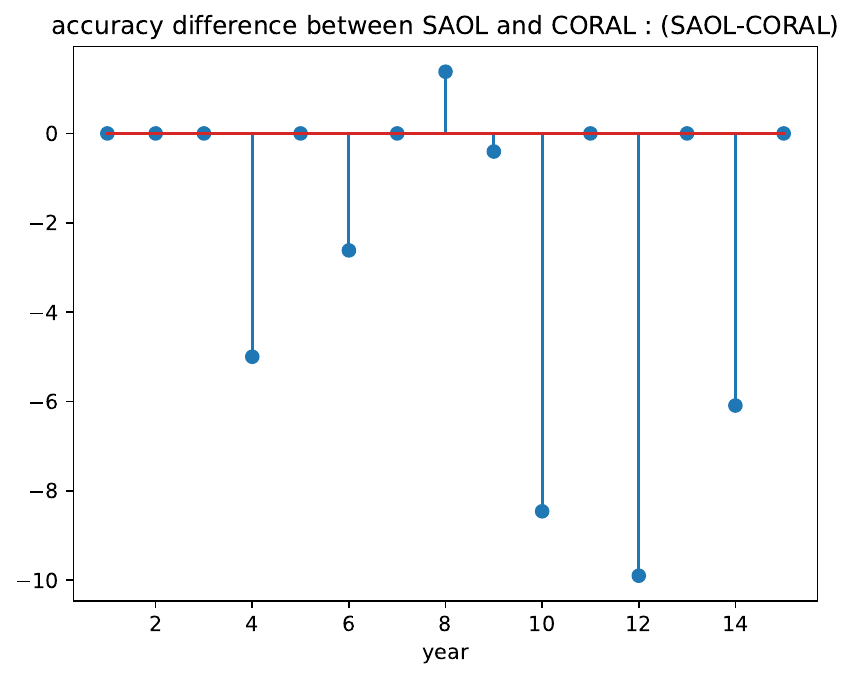}}\hfill
    \subfloat[Huffpost]{\includegraphics[width=0.25\linewidth]{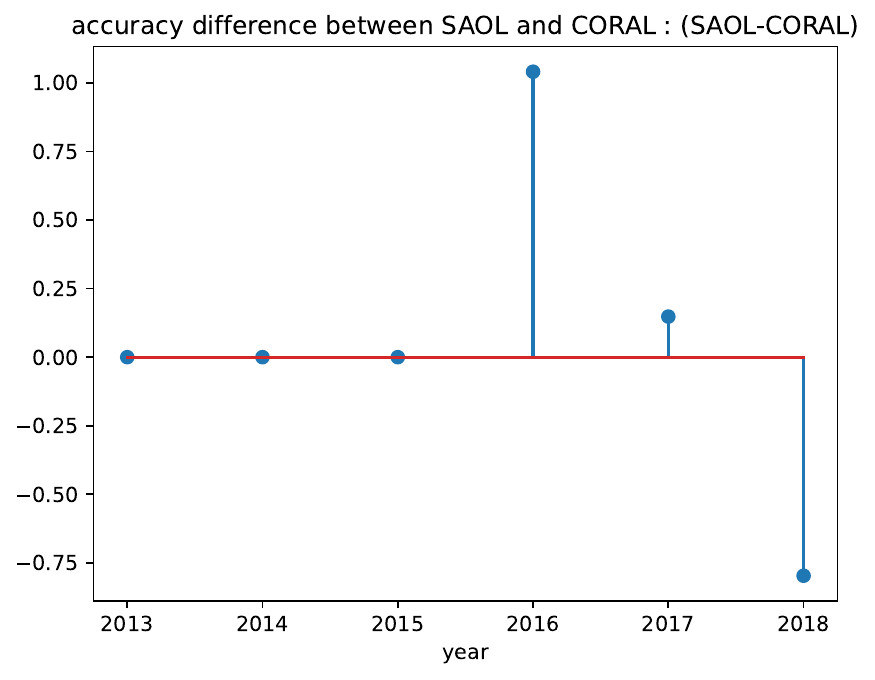}}\hfill
    \subfloat[Arxiv]{\includegraphics[width=0.25\linewidth]{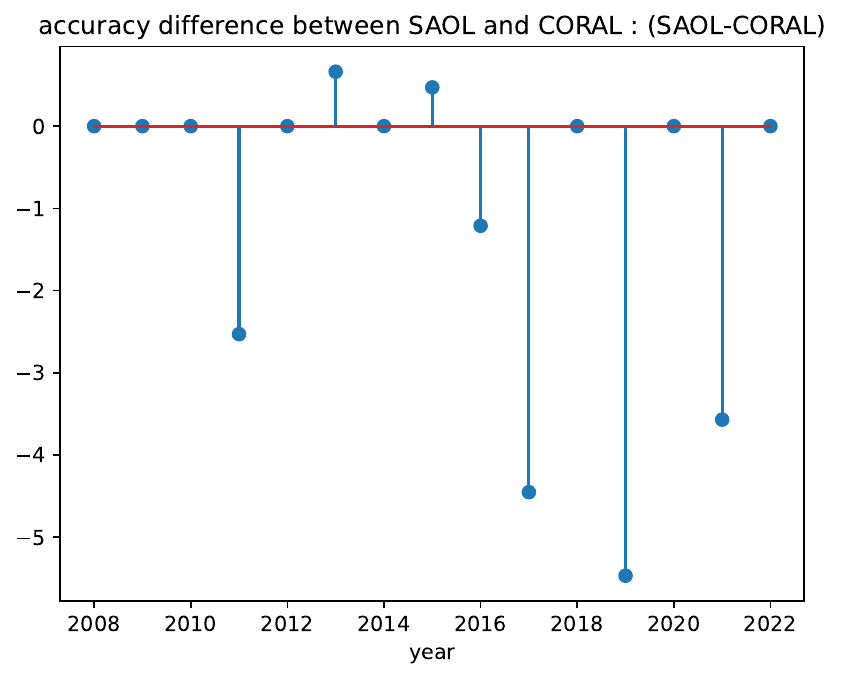}}

    \caption{$\%$ accuracy differences across various timestamps when SAOL is run with CORAL as the online learning algorithm.}
    \label{fig:images_coral_gc}
\end{figure*}

\begin{figure*}[h]
    \centering
    \subfloat[FMOW]{\includegraphics[width=0.25\linewidth]{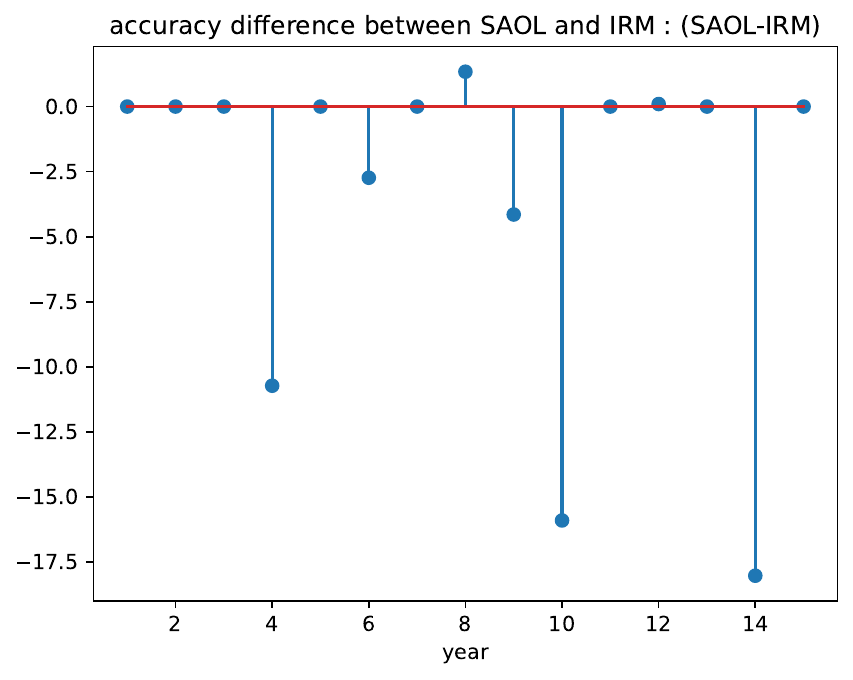}}\hfill
    \subfloat[Huffpost]{\includegraphics[width=0.25\linewidth]{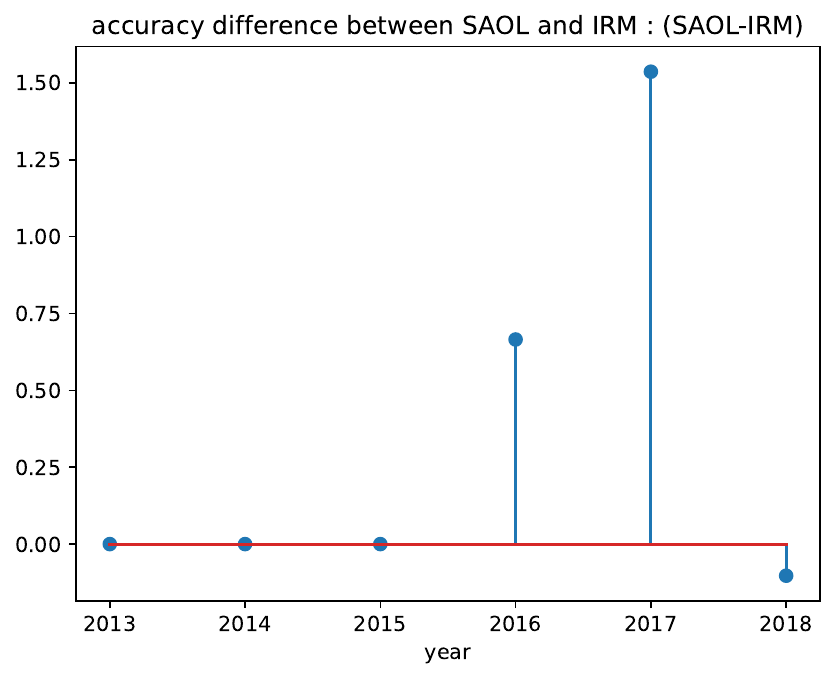}}\hfill
    \subfloat[Arxiv]{\includegraphics[width=0.25\linewidth]{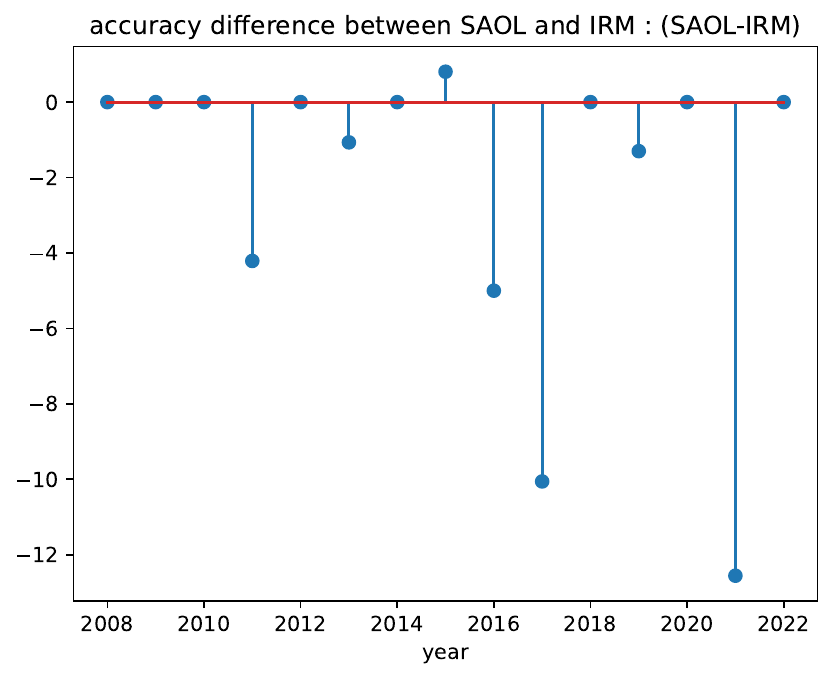}}

    \caption{$\%$ accuracy differences across various timestamps when SAOL is run with IRM as the online learning algorithm.}
    \label{fig:images_irm_gc}
\end{figure*}

\begin{figure*}[h]
    \centering
    \subfloat[FMOW]{\includegraphics[width=0.25\linewidth]{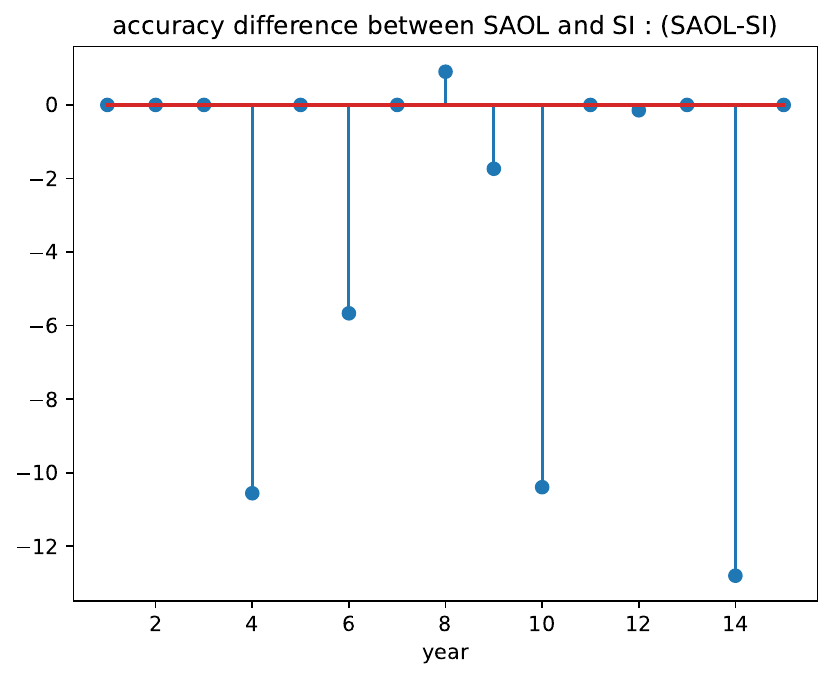}}\hfill
    \subfloat[Huffpost]{\includegraphics[width=0.25\linewidth]{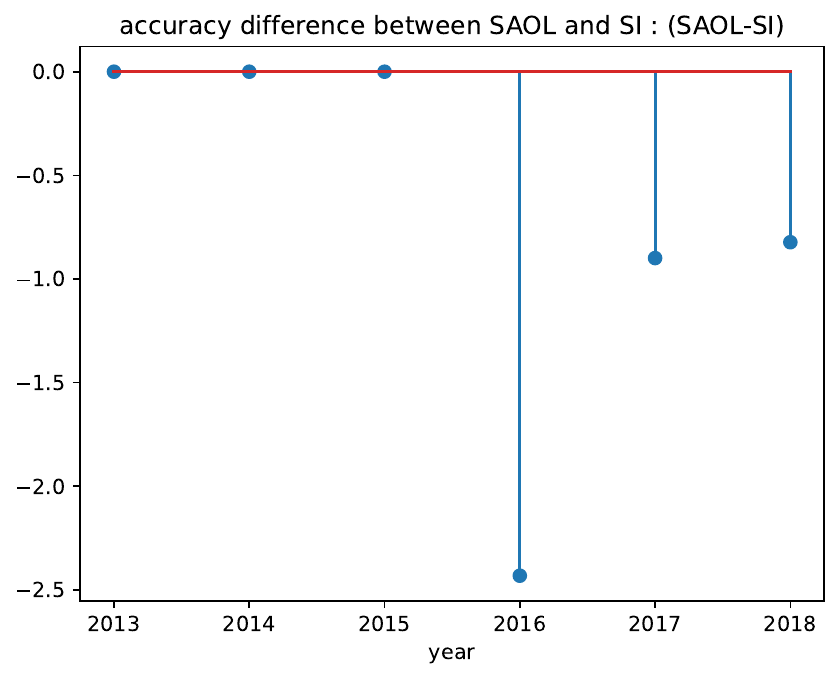}}\hfill
    \subfloat[Arxiv]{\includegraphics[width=0.25\linewidth]{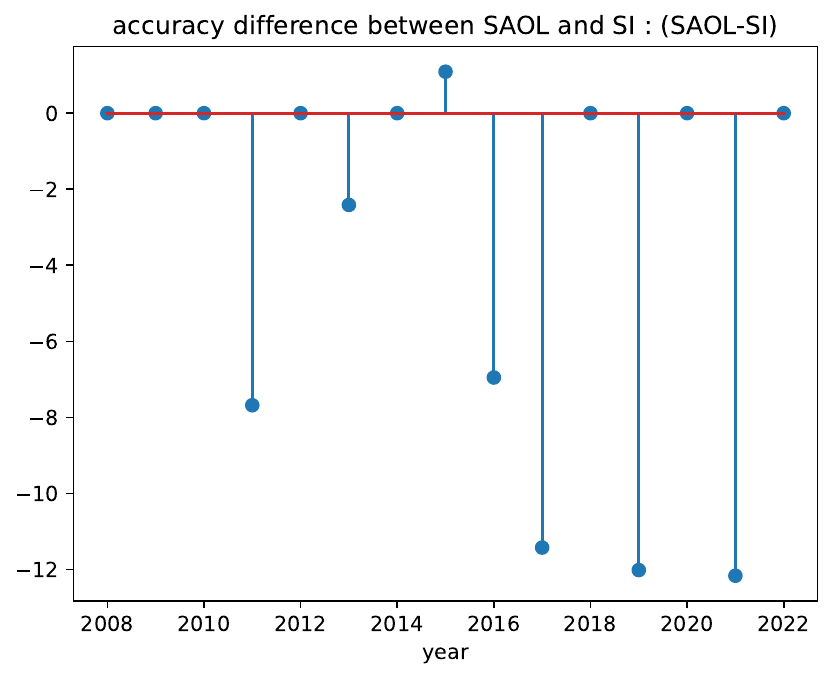}}

    \caption{$\%$ accuracy differences across various timestamps when SAOL is run with SI as the online learning algorithm.}
    \label{fig:images_si_gc}
\end{figure*}

\begin{table*}[]
\centering
\begin{tabular}{|c|c|c|c|}
\hline
\multirow{2}{*}{\begin{tabular}[c]{@{}c@{}}Input\\ OL\end{tabular}} & FMOW (SAOL)                                                   & Huffpost (SAOL)                                                & Arxiv (SAOL)                                                  \\ \cline{2-4} 
                                                                    & $\Delta$acc \%                                               & $\Delta$acc \%                                                & $\Delta$acc \%                                               \\ \hline
SI                                                                  & \begin{tabular}[c]{@{}c@{}}$-2.47$\\ $ \pm 0.082$\end{tabular} & \begin{tabular}[c]{@{}c@{}}$-0.17$\\ $ \pm 0.022$\end{tabular} & \begin{tabular}[c]{@{}c@{}}$-2.95$\\ $\pm 0.093$\end{tabular} \\ \hline
FT                                                                  & \begin{tabular}[c]{@{}c@{}}$-2$\\ $ \pm 0.077$\end{tabular}    & \begin{tabular}[c]{@{}c@{}}$0.03$\\ $\pm 0.014$\end{tabular}   & \begin{tabular}[c]{@{}c@{}}$-3.04$\\ $\pm 0.024$\end{tabular} \\ \hline
IRM                                                                 & \begin{tabular}[c]{@{}c@{}}$-4.06$\\ $ \pm 0.108$\end{tabular} & \begin{tabular}[c]{@{}c@{}}$0.18$\\ $\pm 0.034$\end{tabular}   & \begin{tabular}[c]{@{}c@{}}$-2.34$\\ $\pm 0.083$\end{tabular} \\ \hline
EWC                                                                 & \begin{tabular}[c]{@{}c@{}}$2.05$\\ $ \pm 0.078$\end{tabular}  & \begin{tabular}[c]{@{}c@{}}$0.06$\\ $\pm 0.013$\end{tabular}   & \begin{tabular}[c]{@{}c@{}}$-2.95$\\ $\pm 0.093$\end{tabular} \\ \hline
CORAL                                                               & \begin{tabular}[c]{@{}c@{}}$-0.40$\\ $ \pm 0.034$\end{tabular} & \begin{tabular}[c]{@{}c@{}}$-0.27$\\ $\pm 0.028$\end{tabular}  & \begin{tabular}[c]{@{}c@{}}$-0.48$\\ $\pm 0.038$\end{tabular} \\ \hline
\end{tabular}
\caption{Performance statistics for image (FMOW dataset) and text (Huffpost \& Arxiv datasets) modalities. We report the difference in average classification accuracy ($\%$) across all timestamps obtained by the majority voting-based black-box scheme minus that of the input OL. As discussed before, the majority voting-based variant of \texttt{AWE} can often degrade the performance, signifying the advantage of model selection via refined accuracy estimates.} 
\label{tab:vote}
\end{table*}

\begin{table*}[]
\centering
\begin{tabular}{|c|c|c|c|}
\hline
      & FMOW     & Huffpost & Arxiv    \\ \hline
SI    & $13/0/2$ & $5/0/1$  & $13/0/2$ \\ \hline
FT    & $14/0/1$ & $5/0/1$  & $12/0/3$ \\ \hline
IRM   & $9/1/5$  & $6/0/0$  & $8/0/7$  \\ \hline
EWC   & $12/2/1$ & $4/0/2$  & $12/0/3$ \\ \hline
CORAL & $15/0/0$ & $6/0/0$  & $15/0/0$ \\ \hline
\end{tabular}
\caption{The table summarizes win/draw/lose statistics for the \texttt{AWE} algorithm. We say that a win (draw/lose) happens at a time stamp if the accuracy of \texttt{AWE} is higher (equal/lower) than the input OL algorithm. We can see that in most cases, the fraction of timestamps where \texttt{AWE} does not degrade the accuracy of the base OL is well above $50\%$. This signifies the efficacy of \texttt{AWE} in optimizing the instantaneous regret at each round.} 
\label{tab:awe_win}
\end{table*}

\textbf{Experiments with Voting. }
Besides Eq.\eqref{eq:exp_weights} majority voting is a commonly used model combination scheme in stationary problems mainly due to its computational and statistical efficiencies. We run experiments where instead of using the map given by Eq.\eqref{eq:exp_weights} in \texttt{AWE}, at any round we output a prediction that is recommended by the majority of instances in $\text{ACTIVE}(t)$. The experimental results are reported in Tables \ref{tab:vote} and \ref{tab:other_win} and Figures \ref{fig:images_ft_vote}-\ref{fig:images_si_vote}. As the experimental results show, a map that does not take into account the refined accuracy estimation can often lead to performance degradation. This provides evidence on the efficacy of more nuanced aggregation methods that also take into account the accuracy of the instances as in \texttt{AWE}. We remind the readers that Theorem \ref{thm:mri} guarantees existence of at-least one model in the instance pool that has seen sufficient amount of data from the most recent distribution. However, the majority of instances can still have bad accuracy. Consequently, a majority voting strategy based model combination leads to poor performance. 

\begin{table*}[]
\centering
\begin{tabular}{|c|ccc|ccc|}
\hline
\multirow{2}{*}{} & \multicolumn{3}{c|}{Voting}                                            & \multicolumn{3}{c|}{SAOL}                                              \\ \cline{2-7} 
                  & \multicolumn{1}{c|}{FMOW}     & \multicolumn{1}{c|}{Huffpost} & Arxiv    & \multicolumn{1}{c|}{FMOW}    & \multicolumn{1}{c|}{Huffpost} & Arxiv   \\ \hline
SI                & \multicolumn{1}{c|}{$3/0/12$} & \multicolumn{1}{c|}{$3/0/3$}  & $2/0/13$ & \multicolumn{1}{c|}{$1/8/6$} & \multicolumn{1}{c|}{$0/3/3$}  & $1/8/6$ \\ \hline
FT                & \multicolumn{1}{c|}{$4/0/11$} & \multicolumn{1}{c|}{$3/0/3$}  & $2/0/13$ & \multicolumn{1}{c|}{$2/8/5$} & \multicolumn{1}{c|}{$0/3/3$}  & $1/8/6$ \\ \hline
IRM               & \multicolumn{1}{c|}{$4/0/11$} & \multicolumn{1}{c|}{$4/0/2$}  & $3/0/12$ & \multicolumn{1}{c|}{$2/8/5$} & \multicolumn{1}{c|}{$2/3/1$}  & $1/8/6$ \\ \hline
EWC               & \multicolumn{1}{c|}{$3/0/12$} & \multicolumn{1}{c|}{$3/0/3$}  & $1/0/14$ & \multicolumn{1}{c|}{$2/8/5$} & \multicolumn{1}{c|}{$0/3/3$}  & $1/8/6$ \\ \hline
CORAL             & \multicolumn{1}{c|}{$7/0/8$}  & \multicolumn{1}{c|}{$2/0/4$}  & $4/0/11$ & \multicolumn{1}{c|}{$1/8/6$} & \multicolumn{1}{c|}{$2/3/1$}  & $2/8/5$ \\ \hline
\end{tabular}
\caption{Table summarizing the win/draw/lose numbers for majority voting and SAOL algorithm from \cite{daniely2015strongly}. We see that the number of timestamps where the black-box scheme improves the performance of the base OL is significantly lower than that of Table \ref{tab:awe_win}. For the case of Voting, this signifies that a weighting strategy that does not take into account the refined accuracy estimates may not be useful in practice. For the case of SAOL, this signifies that optimizing the cumulative regret (at the rate of $1/\sqrt{|I|}$, where $I$ is the interval size for any window $I$) instead of instantaneous regret does not lead to significant gains in practice. See also section (a) of the notes of technical novelties.} 
\label{tab:other_win}
\end{table*}

\end{document}